%% file: ic_learning.tex
\documentclass[11pt]{article}
\usepackage{fullpage}
\usepackage{amsmath,amsfonts,amsthm,amssymb}
\usepackage{enumitem}
\usepackage{bbm}
\usepackage[numbers]{natbib}
\usepackage[colorlinks=true, citecolor=blue, linkcolor=red]{hyperref}
\usepackage{tikz}
\usepackage{algorithm}
\usepackage[noend]{algpseudocode}
\usepackage{nicefrac}
\usepackage{mathtools}
\usepackage{multirow}
\usepackage{subfigure}
\usepackage[makeroom]{cancel}
\usepackage{bigints}
\usepackage{relsize}
\usepackage[title]{appendix}
\usepackage{tabularx}
\usepackage{pgf,tikz}
\usetikzlibrary{patterns}
\usetikzlibrary{arrows}
\usepackage{mathrsfs}
\usepackage{bm}
\usepackage[linewidth=1pt]{mdframed}

\newenvironment{enumerate*}%
  {\vspace{-2ex} \begin{enumerate} %
     \setlength{\itemsep}{-1ex} \setlength{\parsep}{0pt}}%
  {\end{enumerate}}
 
\newenvironment{itemize*}%
  {\vspace{-2ex} \begin{itemize} %
     \setlength{\itemsep}{-1ex} \setlength{\parsep}{0pt}}%
  {\end{itemize}}
 
\newenvironment{description*}%
  {\vspace{-2ex} \begin{description} %
     \setlength{\itemsep}{-1ex} \setlength{\parsep}{0pt}}%
  {\end{description}}

\DeclareMathOperator*{\E}{\mathbb{E}}

\newcommand{\vp}{\mathbf{p}}

\newcommand{\WSWM}{\texttt{\upshape WSWM}}
\newcommand{\mwu}{\texttt{MWU}}
\newcommand{\hedge}{\texttt{Hedge}}
\newcommand{\elf}{\texttt{ELF}}
\newcommand{\elfx}{\texttt{ELF-X}}

\newcommand{\vw}{\mathbf{w}}

\newcommand{\bpi}{\boldsymbol{\pi}}

\newcommand{\bp}{\bar{p}}

\newcommand{\tpi}{\widetilde{\pi}}
\newcommand{\tbpi}{\boldsymbol{\widetilde{\pi}}}

\newcommand{\hell}{\hat{\ell}}

\newcommand{\tL}{L}
\newcommand{\1}{\mathbbm{1}}
\newcommand{\Bern}{\mathtt{Bern}}

\newcommand{\wsu}{\texttt{WSU}}
\newcommand{\wsuagg}{\texttt{WSU-Aggr}}
\newcommand{\mwuagg}{\texttt{MWU-Aggr}}
\newcommand{\elfxagg}{\texttt{ELF-X-Aggr}}
\newcommand{\wsux}{\texttt{WSU-UX}}
\newcommand\numberthis{\addtocounter{equation}{1}\tag{\theequation}}

\newtheorem{theorem}{Theorem}[section]
\newtheorem{lemma}{Lemma}[section]
\newtheorem{definition}{Definition}[section]

\newtheorem{example}{Example}[section]

\begin{document}

\title{No-Regret and Incentive-Compatible Online Learning}
\author{
    Rupert Freeman\thanks{Microsoft Research NYC, \texttt{rupert.freeman@microsoft.com}}
    \and 
    David M. Pennock\thanks{DIMACS Center, Rutgers University, \texttt{dpennock@dimacs.rutgers.edu}. Part of the work was conducted when the author was a researcher at Microsoft Research NYC.}
    \and 
    Chara Podimata\thanks{Harvard University, \texttt{podimata@g.harvard.edu}. Part of the work was conducted when the author was an intern at Microsoft Research NYC. The author is supported in part under grant No. CCF-1718549 of the National Science Foundation and the Harvard Data Science Initiative.}
    \and 
    Jennifer Wortman Vaughan\thanks{Microsoft Research NYC, \texttt{jenn@microsoft.com}}
}

\maketitle

\begin{abstract}
 We study online learning settings in which experts act strategically to maximize their influence on the learning algorithm's predictions by potentially misreporting their beliefs about a sequence of binary events. Our goal is twofold. First, we want the learning algorithm to be no-regret with respect to the best fixed expert in hindsight. Second, we want incentive compatibility, a guarantee that each expert's best strategy is to report his true beliefs about the realization of each event.  To achieve this goal, we build on the literature on wagering mechanisms, a type of multi-agent scoring rule. We provide algorithms that achieve no regret and incentive compatibility for myopic experts for both the full and partial information settings.  In experiments on datasets from FiveThirtyEight, our algorithms have regret comparable to classic no-regret algorithms, which are not incentive-compatible. Finally, we identify an incentive-compatible algorithm for forward-looking strategic agents that exhibits diminishing regret in practice.
\end{abstract}

\input{1-intro.tex}
\input{2-model}

\input{3-full_info.tex}

\input{4-bandit_info.tex}

\input{5-elf.tex}

\input{6-experiments.tex}

\input{7-conclusion.tex}

\input{acks}

\bibliographystyle{icml2020}
\bibliography{refs,refs2}
\onecolumn
\input{appendix}

\end{document}

%% file: 1-intro.tex
\section{Introduction}

We study an online learning setting in which a learner makes predictions about a sequence of $T$ binary events~\citep{V90,LW94,CBFHHSW97,FS97,V98,ACBFS02}. The learner has access to a pool of $K$ experts, each with beliefs about the likelihood of each event occurring. The standard goal of the learner is to output a sequence of predictions almost as accurate as those of the best fixed expert in hindsight. Such a learner is said to have no regret.

But what if the experts that the learner consults are strategic agents, capable of reporting predictions that do not represent their true beliefs? As pointed out by~\citet{RS17}, when the learner is not only making predictions but also (implicitly or explicitly) evaluating the experts, experts might have incentive to misreport. The Good Judgment Project,\footnote{\url{https://goodjudgment.com}} a competitor in IARPA's Aggregative Contingent Estimation geopolitical forecasting contest, scored individual forecasters and rewarded the top 2\%---dubbed ``Superforecasters''~\cite{TG15}---with perks such as paid conference travel; some are now employed by a spinoff company. Similarly, the website FiveThirtyEight\footnote{\url{https://fivethirtyeight.com/}} not only predicts election results by aggregating different pollsters, but also publicly scores the pollsters, in a way that correlates with the amount of influence that the pollsters have over the FiveThirtyEight aggregate. It is natural to expect that forecasters might respond to the competitive incentive structure in these settings by seeking to maximize the influence that they exert on the learner's prediction.

When an online learning algorithm is designed in such a way that experts are motivated to report their true beliefs, we say it is \emph{incentive-compatible}.  Incentive compatibility is desirable for several reasons. First, when experts do not report truthfully, the learner's prediction may be harmed. Second, learning algorithms that fail incentive compatibility place an additional layer of cognitive burden on the experts, who must now reason about the details of the algorithm and other experts' reports and beliefs in order to decide how to act optimally. To our knowledge, the standard multiplicative-weights-type algorithms fail incentive compatibility, in the sense that experts can sometimes achieve a greater influence on the algorithm's prediction by misreporting their beliefs; we illustrate this explicitly through manipulation examples. Our goal in this work is to design incentive-compatible online learning algorithms without compromising on the quality of the algorithm's predictions. That is, we seek algorithms that are both incentive-compatible and no-regret, for both the full and partial (bandit) information settings.

Towards this goal, we show a novel connection between online learning and \emph{wagering mechanisms}~\citep{LLVCRSP08,LLVCRSP15}, a type of multi-agent scoring rule that allows a principal to elicit the beliefs of a group of agents without taking on financial risk. 
Using this connection, we construct online learning algorithms that are incentive-compatible and incur sublinear regret. For the full information setting, we introduce Weighted-Score Update ($\wsu$), which yields regret $O(\sqrt{T \ln K})$, matching the optimal regret achievable for general loss functions, even without incentive guarantees. For the partial information setting, we introduce Weighted-Score Update with Uniform Exploration ($\wsux$), which achieves regret $O ( T^{2/3} (K \ln K)^{1/3})$. 

We focus primarily on experts that strategize only about their influence at the next timestep. However, we obtain a partial extension for forward-looking experts. Building on a mechanism that was proposed for forecasting competitions~\cite{WFVPK18}, we identify an algorithm, $\elfx$, for the full information setting that is incentive-compatible and achieves diminishing regret in simulations. 

Our theoretical results are supported by experiments on data gathered from an online prediction contest on FiveThirtyEight. Our algorithms achieve regret almost identical to the classic (and not incentive-compatible) Multiplicative Weights Update ($\mwu$)~\cite{FS97} and $\texttt{EXP3}$~\cite{ACBFS02} algorithms in the full and partial information settings respectively, though $\wsu$ falls short of the optimal regret achieved by $\hedge$ for quadratic loss. 

\paragraph{Related Work.} Other work has drawn connections between online learning and incentive-compatible forecasting, particularly in the context of prediction markets~\citep{ACV13, AF11, FDR12, HS14}. Our work is most closely related to that of \citet{RS17}, but differs from theirs in several important ways.  Most crucially, Roughgarden and Schrijvers consider algorithms that maintain \emph{unnormalized} weights over the experts, and they assume that an expert's incentives are only affected by these weights. In our work, incentives are tied to the expert's \emph{normalized} weight---that is, his probability of being selected by the learning algorithm. We argue that normalized weights better reflect experts' incentives in reality, since reputation tends to be relative more than absolute; put another way, doubling the unnormalized weight of every expert should not increase an expert's utility, since his influence over the learner's prediction remains the same. Under Roughgarden and Schrijvers' model, the design problem is fairly simple when the loss function is a proper loss~\cite{RW09}---that is, one that can be elicited by a proper scoring rule~\citep{S71,GR07}, such as the quadratic loss function---and can be solved with a multiplicative weights algorithm. Because of this, they focus primarily on the absolute loss function, which is not a proper loss.  
In contrast, in our model, the design problem is nontrivial even for these ``easier'' proper loss functions.

Conceptually, our work builds on work by~\citet{WFVPK18}, who use competitive scoring rules---a subclass of wagering mechanisms---to design incentive-compatible forecasting competitions. We discuss their work further in Section~\ref{sec:elf}.
Our work also has connections with the work of \citet{OP16}, who introduce a class of \emph{coin-betting} algorithms for online learning. Although \citet{OP16} do not address incentives and do not make a connection with the wagering mechanisms literature, our $\wsu$ algorithm can be interpreted as a coin-betting algorithm.\footnote{In the language of coin betting, in $\wsu$ experts wager an $\eta$ fraction of their wealth on the positive realization of the event, and the actual outcome of each coin flip is the expert's loss minus the weighted average loss of all other experts.} 

%% file: 2-model.tex
\section{Model and Preliminaries}
\label{sec:model}

We consider a setting in which a learner interacts with a set of $K$ experts, each making probabilistic predictions about a sequence of $T$ binary outcomes.\footnote{We focus on binary outcomes to simplify the presentation of our results, but our techniques could be applied more broadly.}  At each round $t \in [T]$, each expert $i \in [K]$ has a private belief $b_{i,t} \in [0,1]$, unknown to the learner, about the outcome for that round.  Both the experts' beliefs and the sequence of outcomes may be chosen arbitrarily, and potentially adversarially.

In the full information setting, each expert reports his prediction $p_{i,t} \in [0,1]$ to the learner. The learner then chooses her own prediction $\bp_t \in [0,1]$ and observes the outcome realization $r_t \in \{0,1\}$. Finally, the learner and the experts incur losses $\ell_t = \ell(\bp_t,r_t)$ and $\ell_{i,t} = \ell(p_{i,t},r_t), \forall i \in [K]$, where $\ell: [0,1] \times \{ 0,1 \} \to [0,1]$ is a bounded loss function.\footnote{The loss function taking values in $[0,1]$ is without loss of generality since any bounded loss function could be scaled.} As is common in the literature, we restrict our attention to algorithms in which the learner maintains a timestep-specific probability distribution $\bpi_t = (\pi_{1,t}, \dots, \pi_{K,t})$ over the experts, and chooses her prediction $\bp_t$ according to this distribution. Unless specified, this means that the learner predicts $\bp_t = p_{i,t}$ with probability $\pi_{i,t}$; some of our results additionally apply when $\bp_t = \sum_{i \in [K]} \pi_{i,t}p_{i,t}$.

Under partial information, the protocol remains the same except that the learner is explicitly restricted to choosing a single expert $I_t$ on each round $t$ (according to distribution $\bpi_t$) and does not observe the predictions of other experts.

The goal of the learner is twofold. First, she wishes to incur a total loss that is not too much worse than the loss of the best fixed expert in hindsight. This is captured using the classic notion of \emph{regret}, given by
\[
R = \E \left[ \sum_{t \in [T]} \ell_t - \min_{i \in [K]} \sum_{t \in [T]} \ell_{i,t} \right],
 \]
where the expectation is taken with respect to randomness in the learner's choice of $\bp_t$.

No-regret algorithms have been proposed in both the full and partial information settings. Many, such as $\hedge$ \citep{FS97} and $\mwu$ \citep{AHK12}, achieve regret of $O(\sqrt{T \ln K})$ for general loss functions by maintaining unnormalized weights $w_{i,t}$ for each expert $i$ that are updated multiplicatively at each timestep. $\hedge$ uses the update rule $w_{i,t+1}=w_{i,t}\exp(-\eta \ell_{i,t})$, while $\mwu$ uses $w_{i,t+1}=w_{i,t}(1 - \eta \ell_{i,t})$ for appropriately chosen values of $\eta$. These weights are then normalized to arrive at the distribution $\bpi_t$. For the case of exp-concave loss functions, such as the quadratic loss, \citet{KW99} showed that by aggregating experts' predictions and tuning $\eta$ appropriately, $\hedge$ can achieve regret $O(\ln K)$.  

For the partial information setting, the $\texttt{EXP3}$ algorithm of \citet{ACBFS02} achieves a regret of $O(\sqrt{T K \ln K})$. \texttt{EXP3} maintains a set of expert weights similar to those of $\hedge$. However, since the learner can only observe the prediction of the chosen expert, she uses an unbiased estimator $\hell_{i,t}$ of each expert $i$'s loss in her updates in place of $\ell_{i,t}$. The update rule then becomes $w_{i,t+1} = w_{i,t} \exp(-\eta \hell_{i,t})$.

The second goal of the learner is to incentivize experts to truthfully report their private beliefs. In our model, at each timestep $t$, each expert $i$ chooses his report $p_{i,t}$ strategically to maximize the probability $\pi_{i,t+1}$ that he is chosen at timestep $t+1$. An algorithm is \emph{incentive-compatible} if experts maximize this probability by reporting $p_{i,t} = b_{i,t}$, irrespective of the reports of the other experts. 

\begin{definition}[Incentive Compatibility]\label{def:online-ic}
An online learning algorithm is \emph{incentive-compatible} if for every timestep $t \in [T]$, every expert $i$ with belief $b_{i,t}$, every report $p_{i,t}$, every vector of reports of the other experts $\vp_{-i,t}$, and every history of reports $(\vp_{t'})_{t'<t}$ and outcomes $(r_{t'})_{t'<t}$, 
\begin{align*}
&\E_{r_t \sim \Bern(b_{i,t})} [ \pi_{i,t+1} |   \left(b_{i,t}, \vp_{-i,t}\right),  r_t, (r_{t'})_{t'<t}, (\vp_{t'})_{t'<t} ]\\
&\geq \E_{r_t \sim \Bern(b_{i,t})}\left[ \pi_{i,t+1} |   \left(p_{i,t}, \vp_{-i,t}\right),  r_t, (r_{t'})_{t'<t}, (\vp_{t'})_{t'<t} \right] .
\end{align*}
where by $r \sim \Bern(b)$ we denote a random variable $r$ taking value $1$ with probability $b$ and $0$ otherwise. 
\end{definition}

Incentive compatibility guarantees that any regret bounds apply not only with respect to the reports of the experts, but also with respect to their beliefs. This notion of regret is often called \emph{strategic regret}, and in general may be higher or lower than standard regret. For an incentive-compatible algorithm, the two notions coincide.

To achieve incentive compatibility, we restrict attention to proper loss functions~\cite{RW09}, referred to in the forecasting literature as proper scoring rules~\citep{M56,S71,GR07}. 
\begin{definition}
	\label{def:proper-loss}
	A loss function $\ell$ is said to be \emph{proper} if 
	\[ \mathbb{E}_{r \sim \Bern(b)}[\ell (p,r)] \ge \mathbb{E}_{r \sim \Bern(b)}[\ell (b,r)], \forall p \neq b . \]
\end{definition}

Restricting attention to proper loss functions, we are guaranteed that an expert who cares only about his expected loss would truthfully report his beliefs. However, this does not apply for experts who care about their probability of being chosen by the learner, as in our setting. Indeed, known online learning algorithms fail to be incentive-compatible even for proper loss functions. We illustrate this in the following example for $\mwu$ with the (proper) quadratic loss function $\ell(p,r)=(p-r)^2$. Here the normalization of weights by the factor  $\sum_{j \in [K]} w_{j,t}$, which depends on both $p_{i,t}$ and $r_t$, can create incentives for agent $i$ to deviate. We note that a similar counterexample can be proved for Gradient Descent too, and we include it in Appendix~\ref{sec:gd-counterexample}.

\begin{example}\label{ex:mwu-not-ic}
Let $\ell(p,r)=(p-r)^2$. Under standard initialization for $\mwu$, $w_{i,1}=1$ for all $i \in [K]$. Suppose that $b_{1,1}=0.5$ and $p_{i,1}=0$ for all $i \in \{ 2, \ldots, K \}$. Then $\mathbb{E}[\pi_{1,2}]$, the expected probability that expert 1 is chosen at time 2 under $\mwu$ with respect to his own beliefs, is
\begin{align*} 
	0.5 \left( \frac{1-\eta (1-p_{1,1})^2}{K-\eta (1-p_{1,1})^2-\eta (K\!-\!1)} \right) + 0.5 \left( \frac{1\!-\!\eta p_{1,1}^2}{K\!-\!\eta p_{1,1}^2} \right)\!.
\end{align*}
For $K \geq 3$ and $T \geq 9\ln(3)$, the denominator in the first term is less than the denominator in the second term, independent of $p_{1,1}$. The derivative of $\mathbb{E}[\pi_{1,2}]$ with respect to $p_{1,1}$ is therefore strictly positive at $0.5$, implying that expert 1 maximizes his utility by reporting some $p_{1,1}>0.5$.
\end{example}

Thus, unlike in the setting of \citet{RS17}, using a proper loss function with a standard algorithm is not enough, and new algorithmic ideas are needed. To derive our algorithms, we draw a connection between online learning and \emph{wagering mechanisms}, one-shot elicitation mechanisms that allow experts to bet on the quality of their predictions relative to others.  In the one-shot wagering setting introduced by \citet{LLVCRSP08}, each agent $i \in [K]$ holds a belief $b_i \in [0,1]$ about the likelihood of an event. Agent $i$ reports a probability $p_i$ and a wager $w_i \ge 0$. A wagering mechanism, $\Gamma$, maps the reports $\vp = (p_1, \dots, p_K)$, wagers $\vw = (w_1, \dots, w_K)$, and the realization $r$ of the binary event to payments $\Gamma_i(\vp, \vw, r)$ for each agent $i$. The purpose of the wager is to allow each agent to set a maximum allowable loss, which is captured by imposing the constraint that $\Gamma_i(\vp, \vw, r) \ge 0, \forall i \in [K]$.  We restrict our attention to \emph{budget-balanced} wagering mechanisms for which $\sum_{i \in [K]} \Gamma_i(\vp,\vw, r)=\sum_{i \in [K]} w_i$.

A wagering mechanism $\Gamma$ is said to be \emph{incentive-compatible} if for every agent $i \in [K]$ with belief $b_i \in [0,1]$, every report $p_i \in [0,1]$, every vector of reports of the other agents $\vp_{-i}$, and every vector of wagers $\vw$,
$\E_{r \sim \Bern(b_i)} \left[  \Gamma_i\left(\left(b_i, \vp_{-i}\right), \vw, r\right) \right] 
\geq \E_{r \sim \Bern(b_i)}\left[   \Gamma_i\left(\left(p_i, \vp_{-i}\right), \vw, r\right) \right]$.

One class of budget-balanced, incentive-compatible wagering mechanisms is the Weighted Score Wagering Mechanisms  ($\WSWM$s) of \citet{LLVCRSP08,LLVCRSP15}. Fixing any proper loss function $\ell$ bounded in $[0,1]$, agent $i$ receives
\[
\Gamma^{\WSWM}_i(\vp, \vw, r) =  w_i \! \left(\! 1 \! - \! \ell(p_i,r) \! + \! \sum_{j \in [K]}w_j \ell(p_j,r) \! \right).
\]
$\WSWM$s are incentive-compatible because the payment an agent receives is a linear function of his loss, measured by a proper loss function.
An agent makes a profit (i.e., receives payment greater than his wager), whenever his loss is smaller than the wager-weighted average agent loss, so accurate agents are more likely to increase their wealth.

%% file: 3-full_info.tex
\section{The Full Information Setting}\label{sec:full-info}

In this section, we present and analyze an online prediction algorithm, Weighted-Score Update ($\wsu$), for the full information setting.  We show that $\wsu$ is incentive-compatible and achieves regret $O(\sqrt{T \ln K})$. 

Our key observation is that we can define a black-box reduction that transforms any budget-balanced wagering mechanism $\Gamma$ to an online learning algorithm by setting $\bpi_{t+1} = \Gamma( \vp_t, \bpi_t, r_t)$. Here we can interpret an expert's weight according to distribution $\bpi_t$ as their currency. Each expert ``wagers'' $\bpi_t$ at time $t$ and receives a payoff $\bpi_{t+1}$, which depends on the reports of the experts $\vp$ and the realization $r_t$. It is easy to see that any online prediction algorithm that is derived from an incentive-compatible wagering mechanism will in turn be incentive-compatible, because any misreport that increases weight $\bpi_{t+1}$ would also be a successful misreport in the wagering setting.

One might hope that applying this reduction to the $\WSWM$ would directly yield a no-regret online learning algorithm. But this is not the case, due to the fact that an expert who makes an inaccurate prediction can lose too much of his wealth (probability) if all other experts have low loss, and it can take a long time to recover from this. To handle this, we allow experts to ``wager'' only an $\eta$ fraction of their current probability at each timestep for some $\eta \in (0, 0.5]$. This guarantees that no expert can obtain a probability $\pi_{i,t}$ close to zero without having made a long series of inaccurate predictions. Formally, the update rule of our algorithm, the Weighted-Score Update ($\wsu$), is defined by:
\begin{equation}
\label{eqn:wswm-update}
 \pi_{i,t+1} = \eta \Gamma^{\WSWM}_i (\vp_t,\bpi_t, r_t) + (1-\eta) \pi_{i,t} ,
\end{equation}
with weights $\pi_{i,1}$ initialized to $\pi_{i,1}$ = $1/K$ for all $i$.

We must show that $\bpi_{t}$ is a valid probability distribution over experts at each $t$. This follows from the $\WSWM$ being budget-balanced; the proof is in the appendix (Lemma~\ref{lem:prop-distr-wsu}).

By rewriting the $\wsu$ update rule in terms of relative loss $\tL_{i,t} = \ell_{i,t}- \sum_{j \in [K]} \pi_{j,t} \ell_{j,t}$, we can see that the form of the update is quite familiar. In particular, from Equation~\eqref{eqn:wswm-update},
\begin{align*}
\pi_{i,t+1}    & = \eta \pi_{i,t} \left(\! 1 \! - \! \ell_{i,t} \! + \! \sum_{j \in [K]} \pi_{j,t} \ell_{j,t} \! \right) + (1-\eta) \pi_{i,t}
               = \pi_{i,t} ( 1 - \eta \tL_{i,t} ) . \numberthis{\label{eq:upd}}
\end{align*}
This resembles the update rule for the (unnormalized) weights maintained by $\mwu$, but with the relative loss $\tL_{i,t}$ in place of $\ell_{i,t}$. The \texttt{D-Prod} algorithm of \citet{EKMW08} involves a similar update, but using loss relative to a single fixed distribution over experts instead of $\bpi_t$. 

We are now ready to prove our guarantees.
The proof of Theorem~\ref{thm:no-regr-myopic2} proceeds in a similar manner to the standard proof that $\mwu$ satisfies no regret. However, our proof is slightly simpler because we do not need to make a distinction between (unnormalized) weights and (normalized) probabilities. We can therefore avoid introducing the standard potential function used in proofs of no regret.

\begin{theorem}\label{thm:no-regr-myopic2}
$\wsu$ is incentive-compatible and for step size $\eta = \sqrt{\ln(K)/T}$ yields regret $R \leq 2\sqrt{T \ln K}$.
\end{theorem}

\begin{proof}
For incentive compatibility, note that from Equation~\eqref{eqn:wswm-update}, $\pi_{i,t+1}$ is a convex combination of a $\WSWM$ payment and $\pi_{i,t}$, which cannot be influenced by $i$'s report at time $t$. Since truthful reporting (at least weakly) maximizes each of these components, it also maximizes the sum.

For the regret, denoting by $i^*$ the best expert in hindsight, 
\begin{align*}
1       &\geq \pi_{i^*,T+1} = \pi_{i^*,T}\left( 1 - \eta \tL_{i^*,T}\right)  
        = \pi_{i^*,1} \prod_{t \in [T]} \left(1 - \eta \tL_{i^*,t} \right) = \frac{1}{K} \prod_{t\in [T]} \left(1 - \eta \tL_{i^*,t} \right) .
\end{align*}
Taking the logarithm for both sides of this inequality, we get 
\begin{align*}
0 
&\geq - \ln K + \sum_{t\in [T]} \ln \left(1 - \eta \tL_{i^*,t} \right)  
\geq - \ln K + \sum_{t\in [T]}\left( -\eta \tL_{i^*,t} - \eta^2 \tL_{i^*,t}^2 \right),
\end{align*} 
where the last inequality comes from the fact that for $x \leq 1/2$, $\ln(1-x) \geq -x -x^2$ (see Lemma~\ref{lem:technical}). Rearranging and dividing both sides by $\eta$ yields 
\begin{equation*}%
-\sum_{t \in [T]} \tL_{i^*,t} \leq \frac{\ln K}{\eta} + \eta \sum_{t \in [T]} \tL_{i^*,t}^2 .
\end{equation*}
Since we have $\sum_{t \in [T]} \tL_{i^*,t} =\sum_{t \in [T]} \ell_{i^*,t} - \sum_{t \in [T]}\sum_{j \in [K]} \pi_{j,t} \ell_{j,t} = -R$, this becomes 
\begin{equation*}
R \leq \frac{\ln K}{\eta} + \eta \sum_{t \in [T]} \tL_{i^*,t}^2 \leq \frac{\ln K}{\eta} + \eta T .
\end{equation*}
Finally, tuning $\eta = \sqrt{\ln (K)/T}$ gives us the result.
\end{proof}

If $T$ is not known in advance, a standard doubling trick~\citep{ACBFS02} can be applied with only a constant factor increase in regret; see Appendix~\ref{app:anytime-wsu} for details.

The regret and incentive-compatibility guarantees of $\wsu$ presented in Theorem~\ref{thm:no-regr-myopic2} hold for all $[0,1]$-bounded proper loss functions $\ell$. If $\ell$ is additionally convex, then these guarantees carry over to a (possibly more practical) variant of $\wsu$, termed $\wsuagg$, that uses the same update rule but sets $\bp_t = \sum_{i \in [K]} \pi_{i,t}p_{i,t}$ rather than choosing a single expert.  Incentive compatibility is immediate.  The regret bound follows from the fact that, by Jensen's inequality,
\[
\sum_{t \in [T]} \ell\left(\sum_{i \in [K]} \pi_{i,t}p_{i,t}, r_t\right) \leq \sum_{t \in [T]} \sum_{i \in [K]} \pi_{i,t}\ell(p_{i,t},r_t).
\]

%% file: 4-bandit_info.tex
\section{The Partial Information Setting}\label{sec:bandit}

The encouraging results of the previous section apply only when the learner has access to the reports of all experts. But what if the learner has only partial information regarding these reports and still wants to incentivize all experts to report their predictions truthfully? In this section, we provide and analyze a novel algorithm, Weighted-Score Update with Uniform Exploration ($\wsux$),  that is simultaneously no-regret and incentive-compatible in the bandit setting in which the learner chooses a single expert $I_t$ at each round and observes only that expert's prediction. We show this algorithm has regret $O(T^{2/3} (K \ln K)^{1/3})$. This guarantee is weaker than that of \texttt{EXP3}, but we see in Section~\ref{sec:experiments} that $\wsux$ can perform similarly to $\texttt{EXP3}$ in practice with the additional advantage of incentive compatibility.

One might think that the standard trick of replacing the loss $\ell_{i,t}$ with an unbiased estimator $\hell_{i,t}$ in the $\wsu$ update rule would suffice in order to guarantee both incentive compatibility and a regret rate of $O(\sqrt{T \ln K})$. Specifically, following \citet{ACBFS02}, we might consider setting $\hell_{i, t} = 0$ for all experts $i \neq I_t$ whose predictions we do not observe, and $\hell_{I_t,t} = \ell_{I_t,t}/\pi_{{I_t},t}$ for the chosen expert. However, since these estimated losses are unbounded, this could lead to weights $\pi_{i,t}$ moving outside of $[0,1]$, and we would no longer have a valid algorithm. 

To solve this, we mix a distribution generated via $\wsu$-style updates with a small amount of the uniform distribution. This does not affect incentives, since the experts have no way of altering the uniform distribution, and has the convenient property that the estimated loss function is now bounded. By carefully tuning parameters, we are able to guarantee a valid probability distribution over experts.  The resulting updates are given in Algorithm~\ref{algo:bandit}.

\begin{algorithm}[h!]
\caption{$\wsux$ with parameters $\eta$ and $\gamma$ such that $0 < \eta, \gamma < 1/2$ and $\eta K/\gamma \leq 1/2$.}\label{algo:bandit}
\begin{algorithmic}[1]
\State \label{step:1}Set $\pi_{i,1} = \frac{1}{K}, \forall i \in [K]$
\For{$t \in [T]$}
\State \label{step:prob} Choose expert $I_t \sim \tpi_{i,t} = (1-\gamma) \pi_{i,t} + \frac{\gamma}{K}$
\State Compute: $\hell_{I_t,t} = \frac{\ell_{I_t,t}}{\tpi_{I_t,t}}$ and $\hell_{i,t} = 0, \forall i \neq I_t$
\State \label{step:wswm-mwu}Update $\pi_{i,t+1} \!=\! \pi_{i,t}\! \left(1 \!-\! \eta \left(\hell_{i,t}\! -\! \sum_{j \in [K]} \pi_{j,t} \hell_{j,t} \right) \right)$
\EndFor
\end{algorithmic}
\end{algorithm}

We first prove that this is a valid algorithm, that is, that the distributions $\tbpi_t$ from which an expert is selected are valid, under appropriate settings of $\eta$ and $\gamma$. 

\begin{lemma}\label{lem:prop-distr}
If  $\eta K/\gamma \leq 1/2$, the $\wsux$ weights $\bpi_t$ and $\tbpi_t$ are valid probability distributions for all $t \in [T+1]$.
\end{lemma}

\begin{proof}
We prove this inductively for $\bpi_t$ and $\tbpi_t$ simultaneously.  The base case is trivial since at time $t=1$, $\forall i \in [K], \pi_{i,1} =  \tpi_{i,1} = 1/K$. Now assume that for some $t$ both $\bpi_t$ and $\tbpi_t$ are valid probability distributions.
We distinguish two cases. First, suppose $i \neq I_t$. Then, since $\hell_{i,t} = 0$, the $\wsux$ update rule becomes
\begin{equation*}
\pi_{i,t+1} = \pi_{i,t} \left( 1 - \eta \left(0 - \pi_{I_t,t}\frac{\ell_{I_t,t}}{\tpi_{I_t,t}} \right)\right) \geq 0.
\end{equation*} 
Second, suppose $i = I_t$. Then
\begin{align*}
\pi_{i,t+1} &= \pi_{i,t} \left( 1 - \eta \left(\frac{\ell_{i,t}}{\tpi_{i,t}} - \pi_{i,t}\frac{\ell_{i,t}}{\tpi_{i,t}}\right) \right)  
= \pi_{i,t} \left( 1 - \eta \frac{\ell_{i,t}}{\tpi_{i,t}} \left(1 - \pi_{i,t}\right) \right) \\
&\geq \pi_{i,t} \left( 1 - \frac{\eta}{\tpi_{i,t}}\right) 
\geq \pi_{i,t} \left( 1 - \eta  \frac{K}{\gamma}\right) \geq 0, 
\end{align*}
where the penultimate inequality comes from the fact that $\tpi_{i,t} \geq \gamma/K$, since by the inductive assumption $\pi_{i,t} \geq 0$. The last follows from the assumption that $\eta K/\gamma \leq 1/2$.
Moreover, for the sum of probabilities we get: 
\begin{align*}
&\sum_{i \in [K]} \pi_{i,t+1} = \sum_{i \in [K]} \pi_{i,t} \left( 1 - \eta \left( \hell_{i,t} - \sum_{j \in [K]} \pi_{j,t} \hell_{j,t} \right)\right) \\
&= \sum_{i \in [K]} \pi_{i,t} - \eta \left( \sum_{i \it [K]} \pi_{i,t}\hell_{i,t} - \sum_{i \in [K]} \pi_{i,t}\sum_{j \in [K]} \pi_{j,t} \hell_{j,t} \right) \\
&= 1 - \eta \left( \sum_{i \in [K]} \pi_{i,t}\hell_{i,t} - \sum_{j \in [K]} \pi_{j,t} \hell_{j,t} \right) = 1 .
\end{align*}
Thus $\bpi_{t+1}$ is valid. Since $\tbpi_{t+1}$ is a convex combination of two probability distributions, it is also a probability distribution, completing the inductive argument.
\end{proof}

We are now ready to state the main theorem. The requirement that $T \ge K \ln{K}$ ensures that the precondition of Lemma~\ref{lem:prop-distr} is satisfied for the settings of $\eta$ and $\gamma$ used.

\begin{theorem}\label{thm:regr-bandit}
For $T \ge K \ln{K}$ and parameters $\eta = \left(\frac{\ln K}{4K^{1/2}T} \right)^{2/3}$ and $\gamma = \left( \frac{K \ln K}{4T} \right)^{1/3}$, $\wsux$ is incentive compatible and yields regret $R \leq 2(4T)^{2/3} (K \ln K)^{1/3}$.
\end{theorem}

The proof of the theorem will follow from a series of claims and lemmas. We first examine the moments of $\hell_{i,t}$ and verify that it is an unbiased estimator of $\ell_{i,t}$; the proof is direct and in Appendix~\ref{app:bandit-known-T}.

\begin{lemma}[Moments]\label{lem:unbiased}
Taking expectation with respect to the choice of expert at round $t$ and keeping all else fixed, $\forall i \in [K], t \in [T], \E_{I_t \sim \tpi_{t}} \left[ \hell_{i,t}\right] = \ell_{i,t}$.
Furthermore, 
\begin{equation}\label{lem:sec-mom}
\E_{I_t \sim \tpi_t} \left[\hell_{i,t}^2 \right] = \frac{\ell_{i,t}^2}{\tpi_{i,t}} \leq \frac{1}{\tpi_{i,t}} .
\end{equation}
\end{lemma}

We next provide a second-order regret bound. It differs from the standard second-order regret bounds presented for bandit algorithms (see e.g., \citet[Chapter~3]{BCB12}) because it relates the ``estimated regret'' of the learner to the second moment of the estimated loss of the best-fixed expert in hindsight. The proof can be found in Appendix~\ref{app:bandit-known-T}. 

\begin{lemma}[Second-Order Bound]\label{eq:sec-order}
For $\wsux$, the probability vectors $\bpi_1, \dots, \bpi_T$ and the estimated losses $\hell_{i,t}$ for $i \in [K], t \in [T]$ induce the following second-order bound:
\begin{align*}
\sum_{t \in [T]} \sum_{i \in [K]} \pi_{i,t} \hell_{i,t} &- \sum_{t \in [T]}\hell_{i^*,t} \leq 
\frac{\ln K}{\eta} + \eta \sum_{t \in [T]} \hell_{i^*,t}^2 + \eta \sum_{t \in [T]} \sum_{i \in [K]} \pi_{i,t} \hell_{i,t}^2
\end{align*}
where $i^* = \arg \min_{i \in [K]} \sum_{t \in [T]} \ell_{i,t}$.
\end{lemma}

\begin{proof}
Since $\bpi_{T+1}$ is a valid probability distribution (Lemma~\ref{lem:prop-distr}), we have
\begin{align*}
1 &\geq\pi_{i^*,T+1} 
        = \pi_{i^*,T} \left(1 \!- \!\eta \left(\hell_{i^*,{T}} \!-\! \sum_{j \in [K]}\pi_{j,{T}} \hell_{j,T} \right) \right) \\
&= \pi_{i^*, 1} \prod_{t \in [T]} \left( 1 - \eta \left(\hell_{i^*,{t}} - \sum_{j \in [K]}\pi_{j,{t}} \hell_{j,t} \right)\right)
\end{align*} 
Taking the logarithm for both sides, and using the fact that $\pi_{i,1} = 1/K, \forall i \in [K]$, we get
\begin{equation}\label{eq:pot-lb}
0 \geq - \ln K + \sum_{t \in [T]} \ln \left( 1 - \eta \left(\hell_{i^*,{t}} - \sum_{j \in [K]}\pi_{j,{t}} \hell_{j,t} \right)\right) .
\end{equation}%
We next show that for all $t \in [T]$ and any $i \in [K]$
$\eta \left(\hell_{i,{t}} - \sum_{j \in [K]}\pi_{j,{t}} \hell_{j,t} \right) \leq 1/2$.
We distinguish two cases. First, if $i \neq I_t$, then the inequality holds since $\hell_{i,t} = 0$ and as a result the expression becomes $- \eta \cdot \pi_{I_t,t}\hell_{I_t,t} \leq 0$.
Second, if $i = I_t$, then the expression becomes
\begin{align*}
&\eta \frac{\ell_{I_t,t}}{\tpi_{I_t,t}} - \eta \pi_{I_t,t}\frac{\ell_{I_t,t}}{\tpi_{I_t,t}} = \eta \frac{\ell_{I_t,t}}{\tpi_{I_t,t}} \left( 1 - \pi_{I_t,t} \right) \\ 
&\leq \eta \frac{1}{\tpi_{I_t,t}}  &\tag{$\pi_{i,t} \geq 0, \ell_{i,t} \leq 1$} \\
&\leq \eta \frac{K}{\gamma} &\tag{$\tpi_{i,t} \geq \gamma/K$, since $\pi_{i,t} \geq 0$} \\
&\leq \frac{1}{2} &\tag{by definition}
\end{align*}
We can now lower bound Equation~\eqref{eq:pot-lb} using the fact that for $z \leq 1/2$ it holds that: $\ln ( 1 - z) \geq -z -z^2$ (Lemma~\ref{lem:technical}).
\begin{align*}
0 &\geq - \ln K + \sum_{t \in [T]} \left[- \eta \left(\hell_{i^*,{t}} - \sum_{j \in [K]}\pi_{j,{t}} \hell_{j,t} \right) \right] - \sum_{t \in [T]} \left[\eta^2 \left(\hell_{i^*,{t}} - \sum_{j \in [K]}\pi_{j,{t}} \hell_{j,t} \right)^2\right]\\
&\geq - \ln K - \eta \left[\sum_{t \in [T]} \hell_{i^*,t} - \sum_{t \in [T]} \sum_{j \in [K]}\pi_{j,t} \hell_{j,t} \right] - \eta^2 \left[\sum_{t \in [T]} \left( \hell_{i^*,t} - \sum_{j \in [K]} \pi_{j,t}\hell_{j,t}\right)^2 \right] \\
&\geq - \ln K - \eta \left[\sum_{t \in [T]} \hell_{i^*,t} - \sum_{t \in [T]} \sum_{j \in [K]}\pi_{j,t} \hell_{j,t} \right] - \eta^2 \sum_{t \in [T]} \hell_{i^*,t}^2 - \eta ^2 \sum_{t \in [T]} \left(\sum_{j \in [K]} \pi_{j,t}\hell_{j,t}\right)^2 \\
&\geq - \ln K - \eta \left[\sum_{t \in [T]} \hell_{i^*,t} - \sum_{t \in [T]} \sum_{j \in [K]}\pi_{j,t} \hell_{j,t} \right] - \eta^2 \sum_{t \in [T]} \hell_{i^*,t}^2 - \eta ^2 \sum_{t \in [T]} \sum_{j \in [K]} \pi_{j,t}\hell_{j,t}^2 
\end{align*}
where the second inequality uses the fact that for $a,b$ non-negative, $(a - b)^2 \leq a^2 + b^2$ and the last inequality uses Jensen's inequality for function $f(x) = x^2$. Rearranging the latter and dividing both sides by $\eta$ gives the result. 
\end{proof}

With that we can complete the proof of Theorem~\ref{thm:regr-bandit}.

\begin{proof}[Proof of Theorem~\ref{thm:regr-bandit}]
It follows from incentive compatibility of $\wsu$ that an expert maximizes the expected value of $\pi_{i,t+1}$ by minimizing the expected value of $\hell_{i,t}$. From the definition of $\hell_{i,t}$, it is easy to see that minimizing the expected value of $\hell_{i,t}$ is equivalent to minimizing the expected value of $\ell_{i,t}$. By properness of $\ell$, this is achieved by truthfully reporting $p_{i,t}=b_{i,t}$. 

We now show the regret bound. Taking expectations with respect to the choice of expert at round $t$ for both sides of the equation in Lemma~\ref{eq:sec-order}, we get 
\begin{align*}
\sum_{t \in [T]} \sum_{i \in [K]} \pi_{i,t} \E_{I_t \sim \tpi_t} \left[\hell_{i,t}\right] - \sum_{t \in [T]} \E_{I_t \sim \tpi_t} \left[\hell_{i^*,t}\right] \leq \eta \sum_{t \in [T]} \E_{I_t \sim \tpi_t} \left[ \hell_{i^*,t}^2 \right]\!+\!\frac{\ln K}{\eta}\!+\!\eta \sum_{t \in [T]} \sum_{i \in [K]} \pi_{i,t} \E_{I_t \sim \tpi_t} \left[ \hell_{i,t}^2 \right]\!.
\end{align*}
Using Lemma~\ref{lem:unbiased}, this gives us
\begin{align*}
\sum_{t \in [T]} &\sum_{i \in [K]} \pi_{i,t} \ell_{i,t} - \sum_{t \in [T]} \ell_{i^*,t} \leq \eta \sum_{t \in [T]} \frac{1}{\tpi_{i^*,t}} + \frac{\ln K}{\eta} + \eta \sum_{t \in [T]} \sum_{i \in [K]} \pi_{i,t} \frac{1}{\tpi_{i,t}}\\
&\leq \eta \sum_{t \in [T]} \frac{K}{\gamma} + \frac{\ln K}{\eta} + 2 \eta K T \leq \frac{\eta K T}{\gamma}  + \frac{\ln K}{\eta} + 2 \eta K T,
\end{align*}
where the second inequality uses the fact that $\pi_{i,t} \leq 2 \tpi_{i,t}, \forall i \in [K], t \in [T]$ since $\gamma/K \geq 0$ and $\gamma \leq 1/2$. Next, we re-write $\pi_{i,t} = \frac{\tpi_{i,t} - \gamma/K}{1 - \gamma}$, yielding
\begin{align*}
\sum_{t \in [T]} \sum_{i \in [K]} &\frac{\tpi_{i,t} - \frac{\gamma}{K}}{1 - \gamma} \ell_{i,t} - \sum_{t \in [T]} \ell_{i^*,t} \leq \leq \frac{\eta K T}{\gamma}  + \frac{\ln K}{\eta} + 2\eta K T.
\end{align*} 
Since $1 - \gamma < 1$ and $\ell(p_{i,t},r_t)\leq 1$, this can be relaxed to:
\begin{align*}
\sum_{t \in [T]} \sum_{i \in [K]} &\tpi_{i,t}\ell\left(p_{i,t},r_t\right) - \sum_{t \in [T]} \ell\left(p_{i^*,t},r_t\right) \leq \\ &\leq \gamma T + \frac{\eta K T}{\gamma} + \frac{\ln K}{\eta} + 2\eta K T.
\end{align*} 

Making $\gamma T = \eta K T/\gamma$ by setting $\gamma = \sqrt{\eta K}$, and $\eta = \left(\frac{\ln K}{4 K^{1/2} T} \right)^{2/3}$ we get the regret result.\footnote{The last derivation requires that $\eta \leq 1/K$, which is true for large enough horizons $T \geq K \ln K$.}
\end{proof}

As in the full information setting, a doubling trick can be applied if $T$ is unknown (Appendix~\ref{app:anytime-wsux}).

We note that, unlike the full information setting in which $\wsu$ achieves the optimal regret bound for general loss functions, our regret bound in Theorem~\ref{thm:regr-bandit} is not as good as what can be achieved without incentive compatibility. 
Examining our analysis, one can see that if the loss of the best-fixed expert in hindsight is zero at each round, then the regret guarantee achieved by $\wsux$ would be the same as \texttt{EXP3}, i.e., $O(\sqrt{T \ln K})$. Closing this gap via a tighter analysis of $\wsux$ or via a new incentive-compatible algorithm is a compelling question for future work.

%% file: 5-elf.tex
\section{Forward-Looking Experts}
\label{sec:elf}
   
So far we have assumed that the experts are myopic, aiming at time $t$ to optimize their influence on the algorithm only at time $t+1$ with no regard for future rounds. It is natural to ask whether it is possible to design learning algorithms that satisfy no regret while incentivizing truthful reports from forward-looking experts who care about their influence $\pi_{i,t'}$ at all $t' > t$. Neither $\wsu$ nor $\wsux$ achieve this goal; see the appendix for examples that illustrate why.\footnote{It is worth noting that in these examples, an expert can gain only a negligible amount from misreporting; it is an open question whether $\wsu$ satisfies some notion of $\epsilon$-incentive compatibility.}

In order to derive an online learning algorithm that is incentive-compatible for forward-looking experts, we build on work by \citet{WFVPK18}, who studied a forecasting competition setting in which agents make predictions about a series of independent events, competing for a single prize. Unlike in our setting, their goal was to derive an incentive-compatible mechanism for choosing the winning agent; they are agnostic to how the elicited forecasts are aggregated. They defined a mechanism, Event-Lotteries Forecaster Selection Mechanism ($\elf$), in which, for every predicted event $\tau$, every agent $i$ is assigned a probability of being the event winner based on the quality of their prediction.  The winner of the competition is the agent who wins the most events.

We build on this idea to define an online learning algorithm, $\elfx$, for the full information setting.  Like $\wsu$, $\elfx$ incorporates $\WSWM$ payments, but in a different way. The distribution $\bpi_t$ at time $t$ is defined as the distribution over experts output by the following randomized process: 
\begin{enumerate*} 
\item At each round $\tau \in [t]$, pick agent $i$ as the ``winner'' $x_\tau$ with probability
\[
\frac{1}{K}\left( 1 - \ell_{i,\tau} + \frac{1}{K}\sum_{j \in [K]} \ell_{j,\tau} \right) .
\]
\item Select $\arg\max_{i \in [K]} \sum_{\tau \in [t]}\1(x_\tau \! =\! i)$, the expert who won the most events, breaking ties uniformly. 
\end{enumerate*}

It can be shown by a similar argument to that of~\citet{WFVPK18} that $\elfx$ is incentive-compatible. The proof, along with a formal definition of incentive compatibility for forward-looking experts, is in the appendix.

\begin{theorem}
	\label{thm:elfx-ic}
	$\elfx$ is incentive-compatible for forward-looking experts.
\end{theorem}

While proving that $\elfx$ is no-regret remains an open problem, in the following section, we present experimental results suggesting that its regret is sublinear in $T$ in practice.

%% file: 6-experiments.tex
\section{Experiments}\label{sec:experiments}

\begin{figure*}[t!]
\centering
\subfigure{\includegraphics[width=0.3\textwidth]{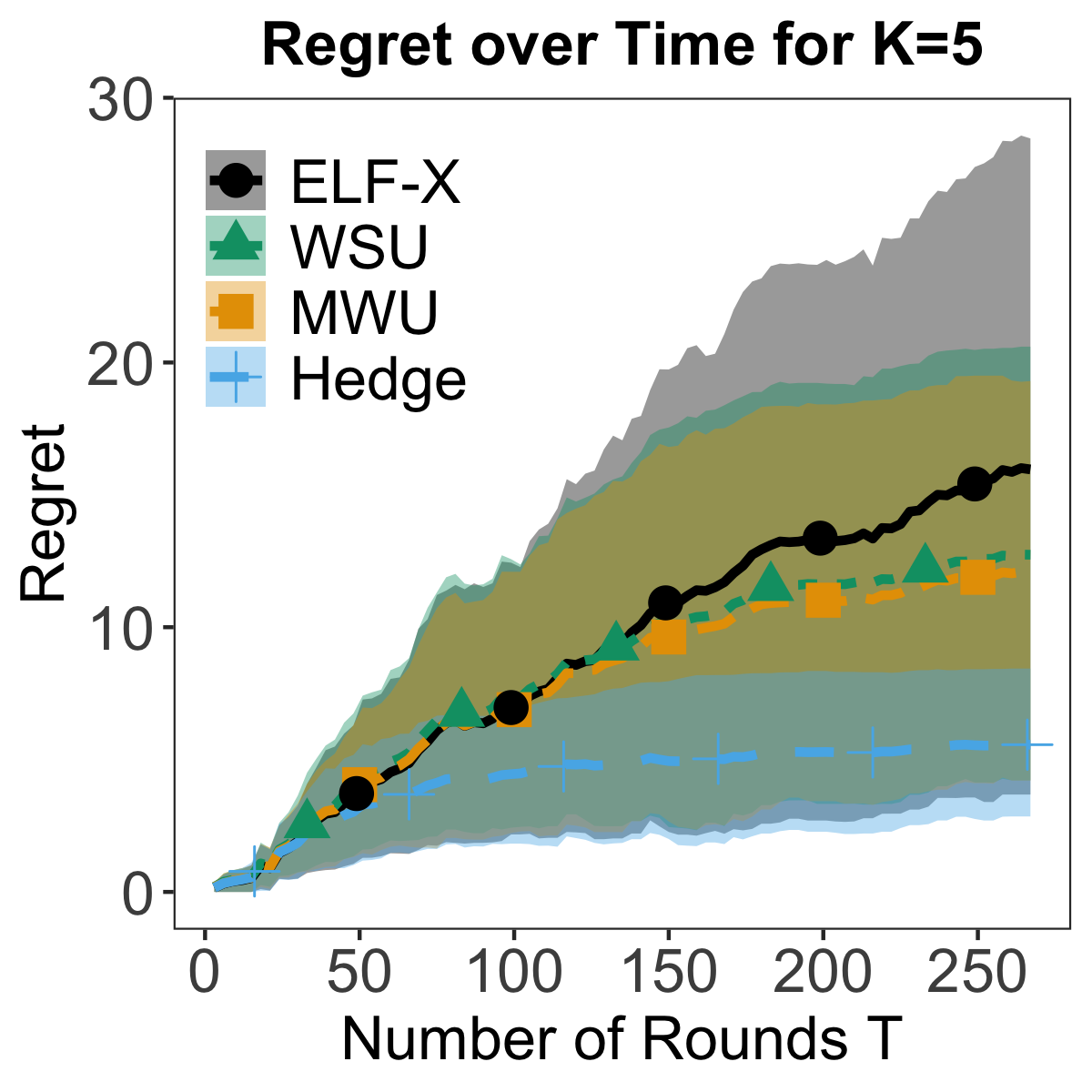}} \hfill
\subfigure{\includegraphics[width=0.3\textwidth]{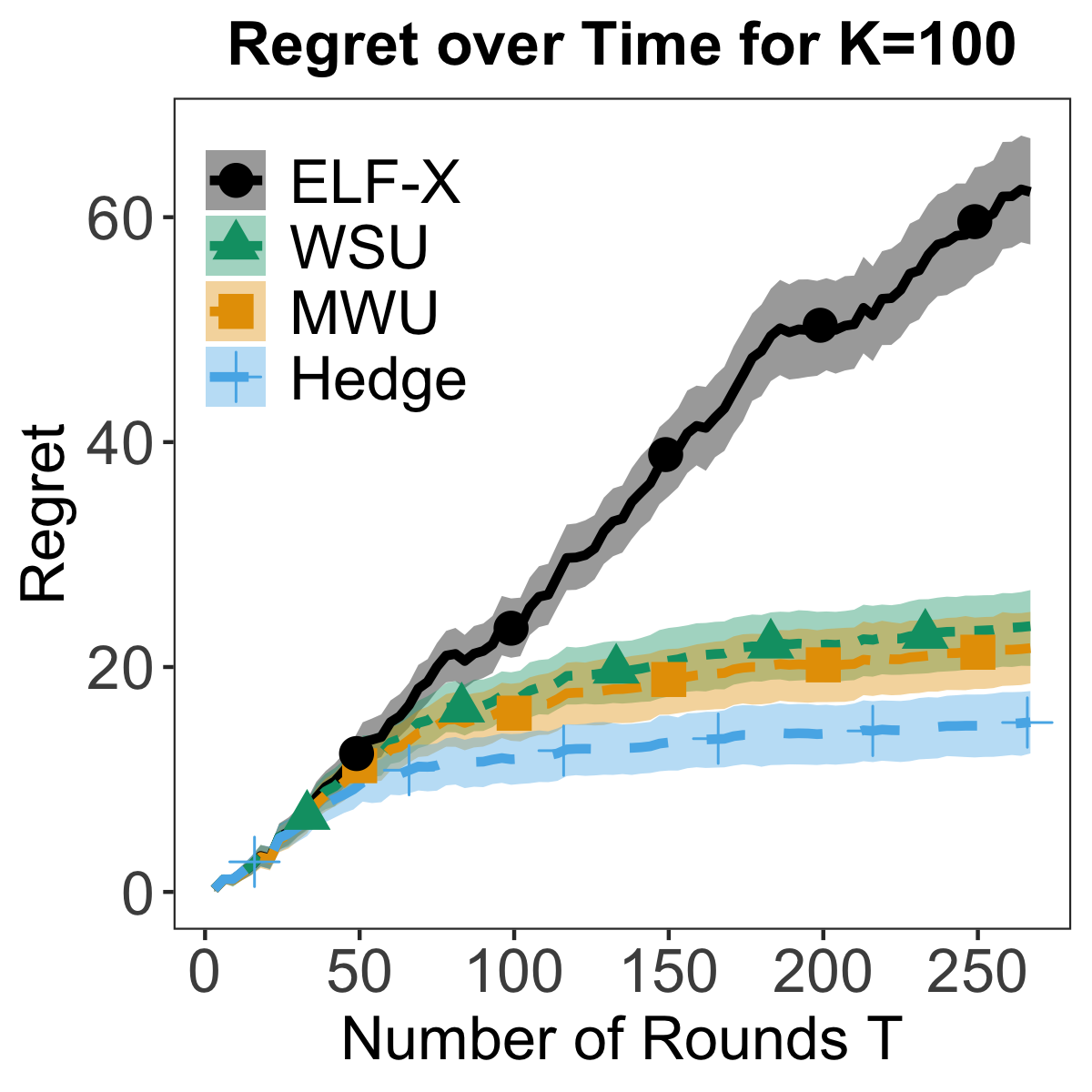}} \hfill
\subfigure{\includegraphics[width=0.3\textwidth]{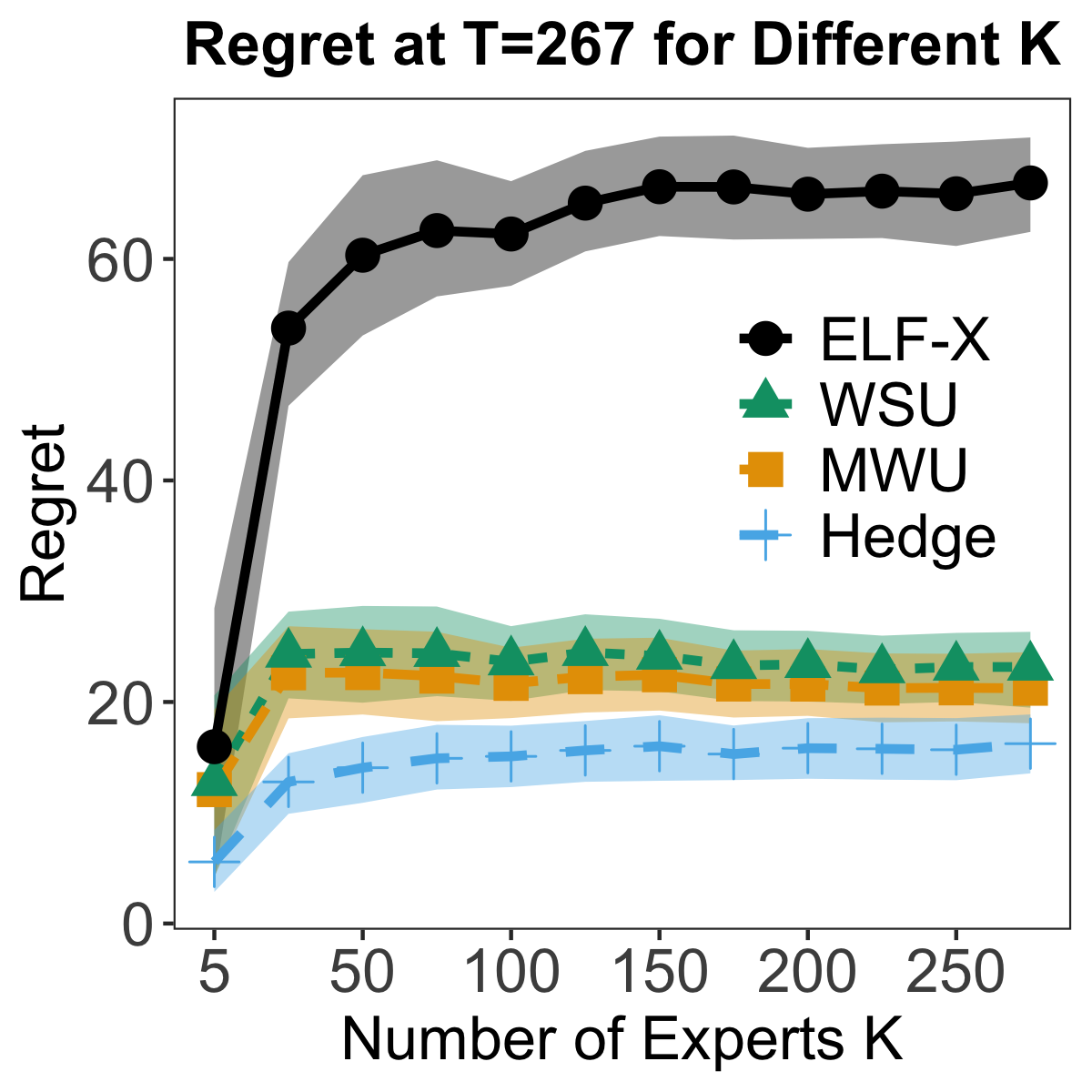}} \\  
\subfigure{\includegraphics[width=0.3\textwidth]{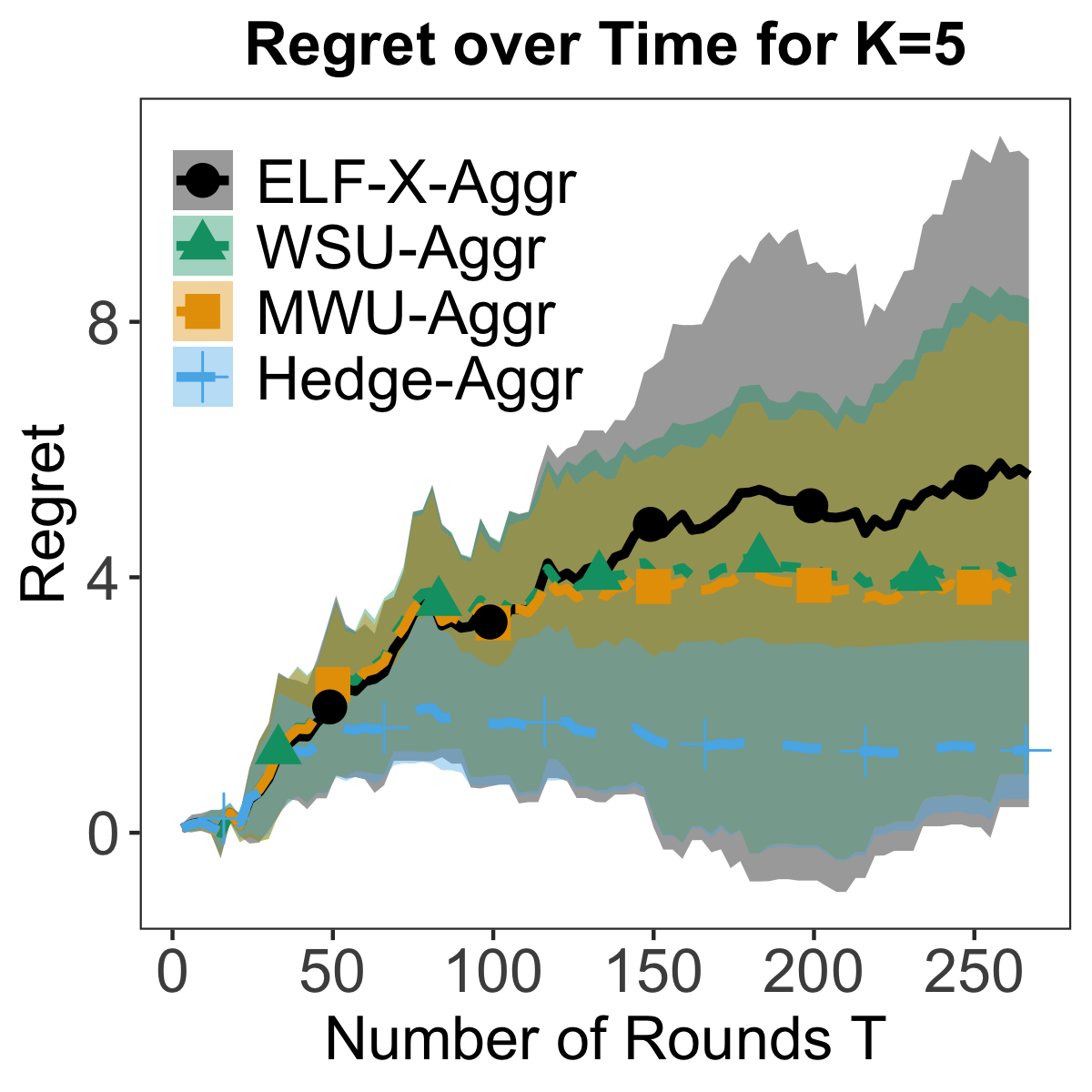}} \hfill
\subfigure{\includegraphics[width=0.3\textwidth]{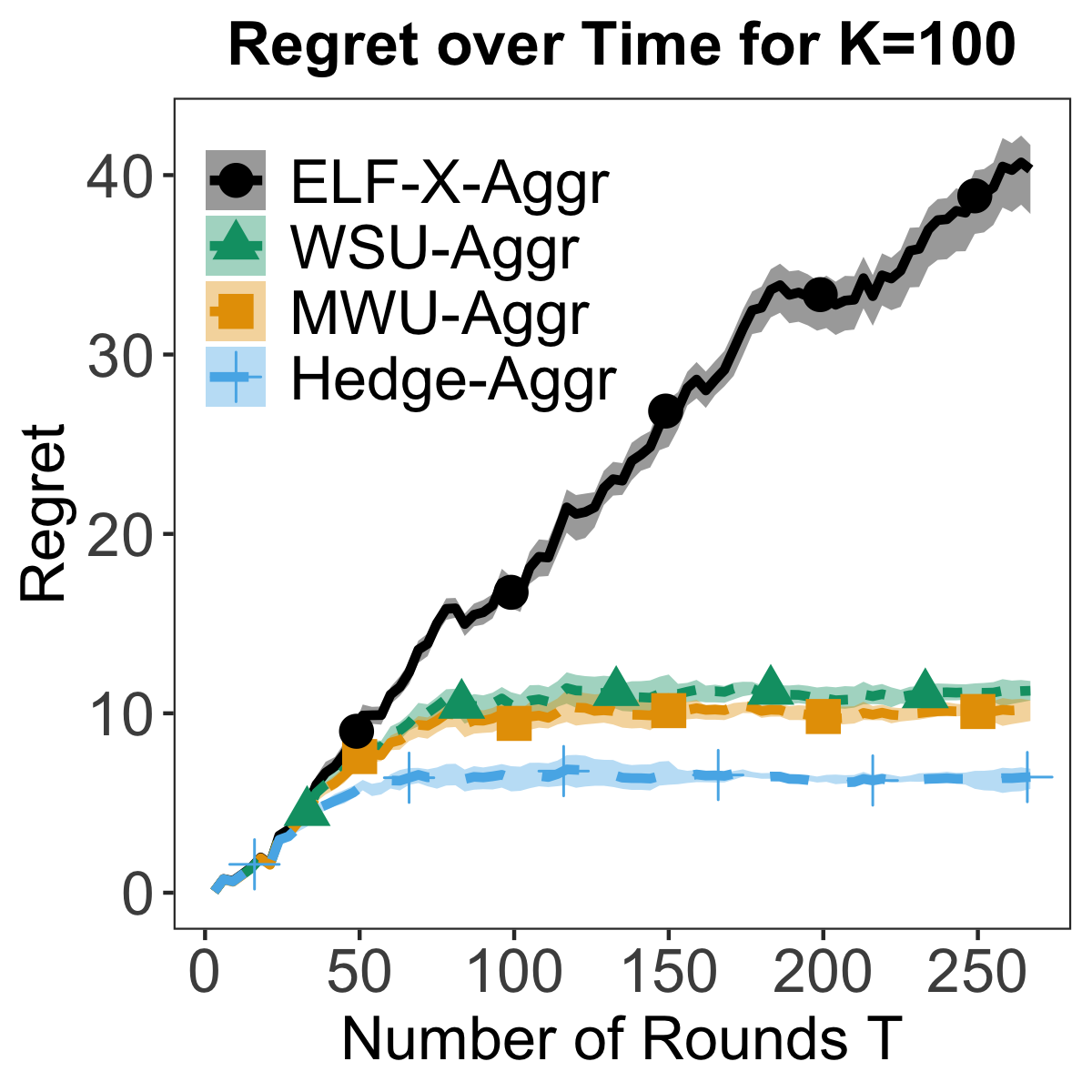}} \hfill
\subfigure{\includegraphics[width=0.3\textwidth]{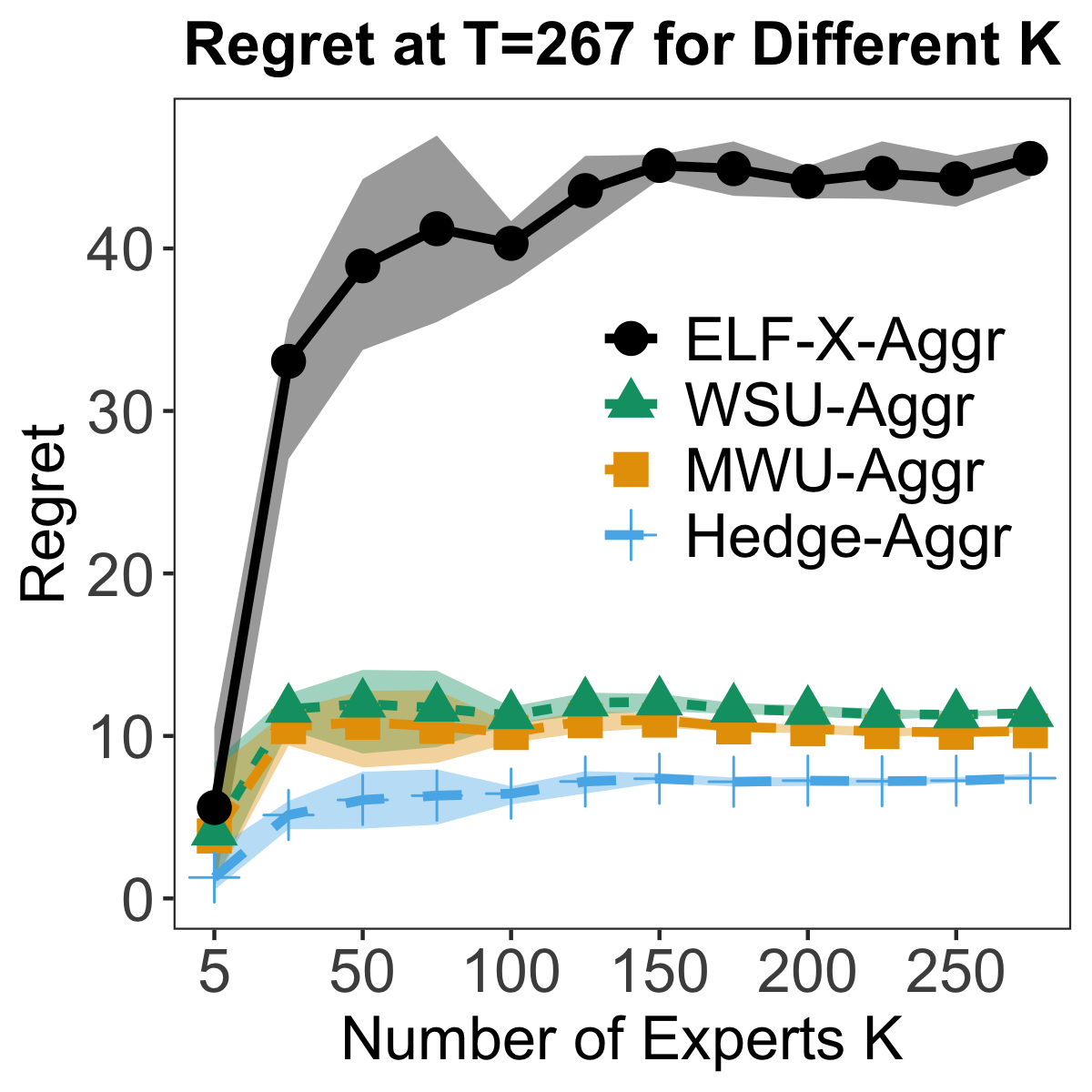}} \\ 
\subfigure{\includegraphics[width=0.3\textwidth]{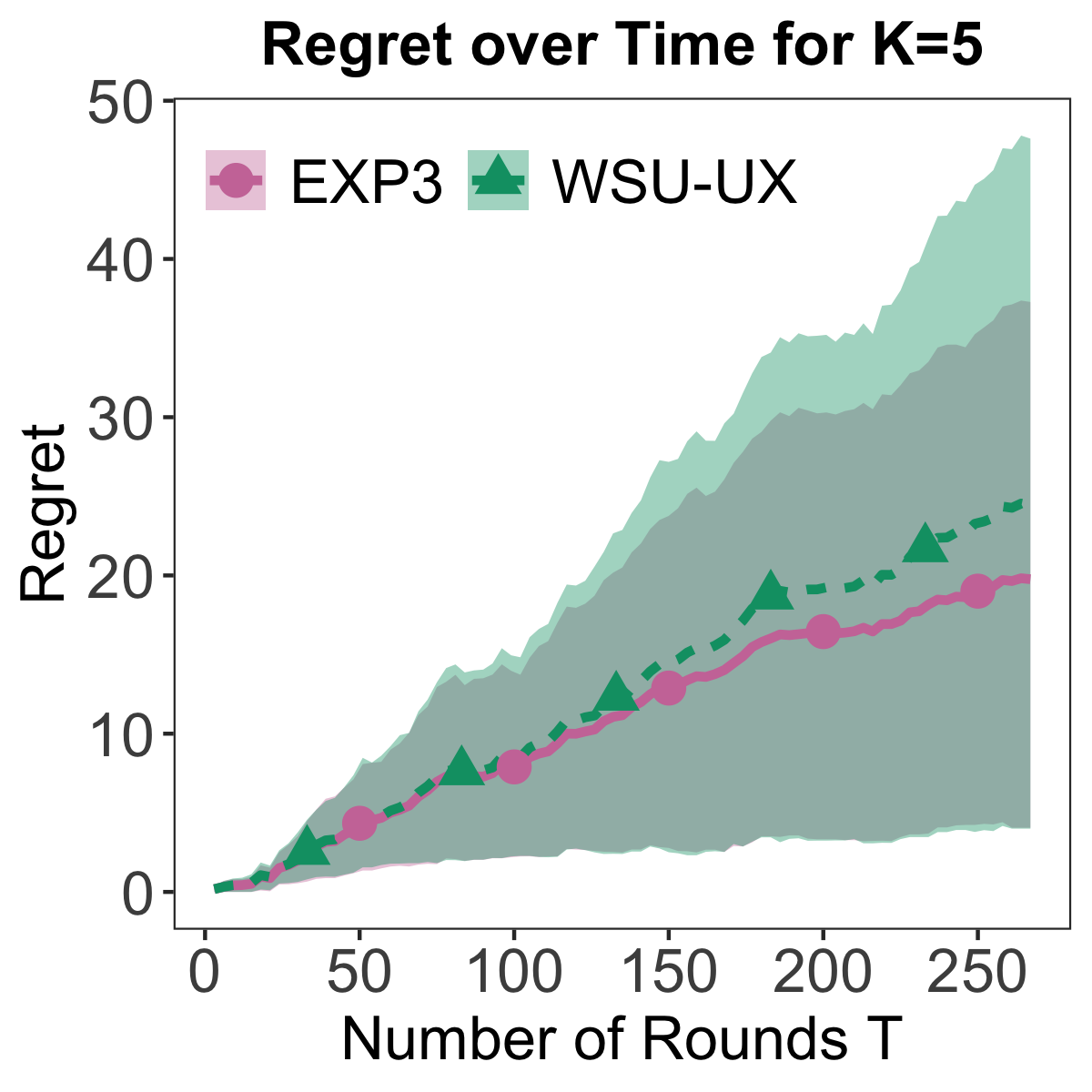}} \hfill
\subfigure{\includegraphics[width=0.3\textwidth]{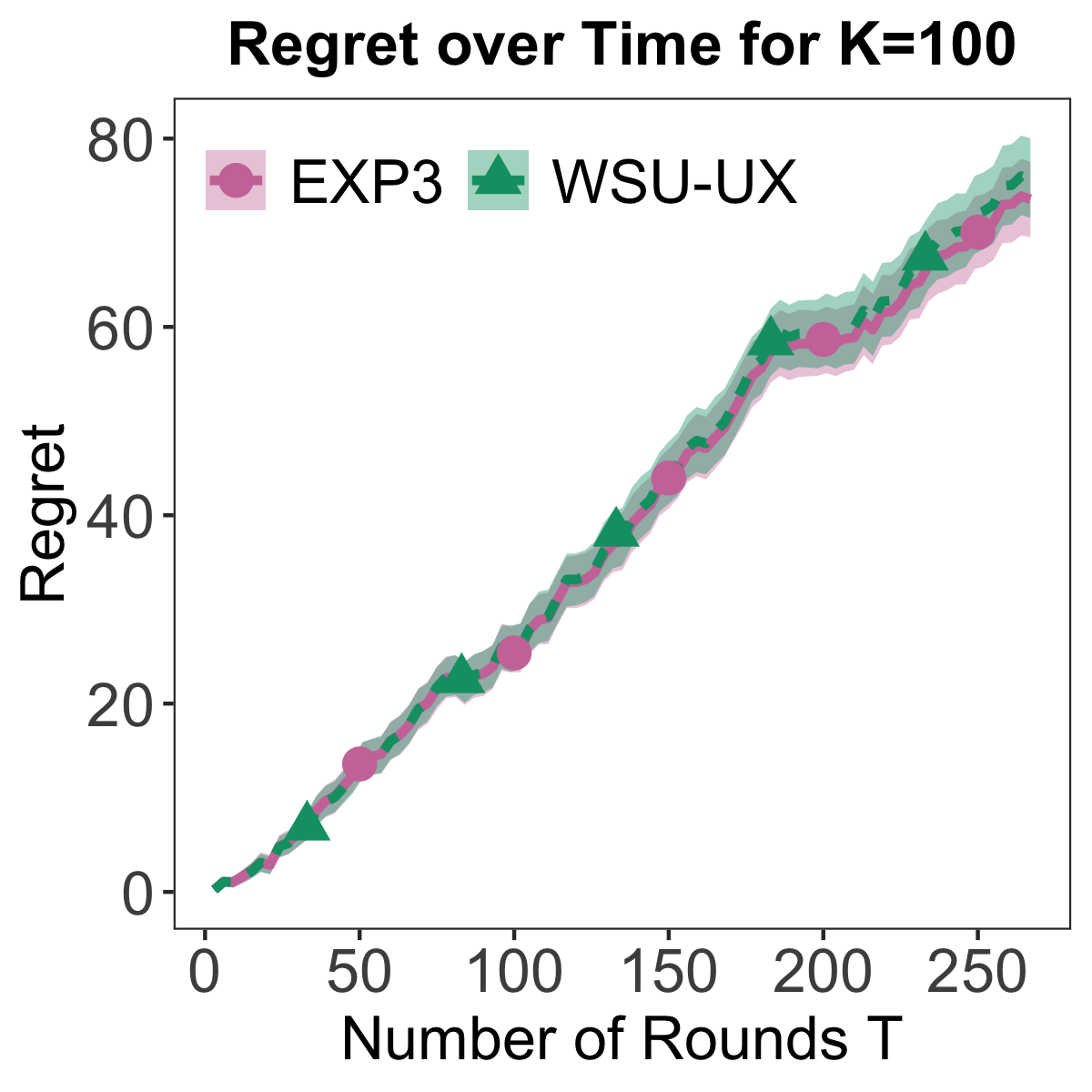}} \hfill
\subfigure{\includegraphics[width=0.3\textwidth]{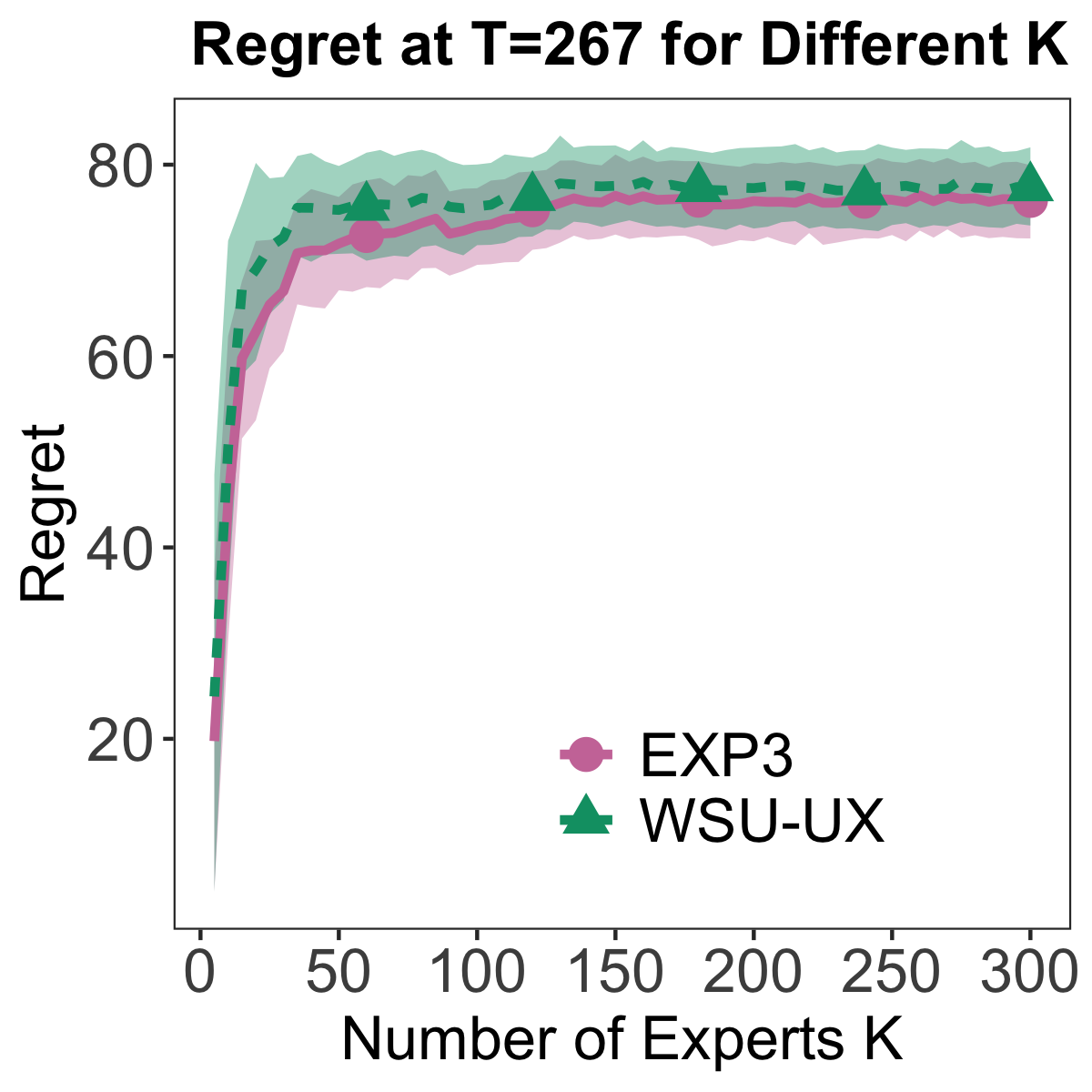}}
\caption{Comparisons on the 2018--2019 FiveThirtyEight NFL dataset. Top: Full-information setting with $\bp_t$ the prediction of a single expert chosen according to $\bpi_t$. Middle: Full-information setting with $\bp_t = \sum_{i \in [K]} \pi_{i,t}p_{i,t}$. Bottom: Partial information setting.}
\label{fig:nfl18-19}
\end{figure*}%

In this section, we empirically evaluate the performance of our proposed incentive-compatible algorithms, $\wsu$ and $\wsux$, compared with standard no-regret algorithms. We also evaluate the performance of $\elfx$, which is incentive-compatible for non-myopic experts. Our code and the datasets we use are publicly available online.\footnote{Code: \small\url{https://github.com/charapod/noregr-and-ic}. Datasets: \url{https://github.com/fivethirtyeight/nfl-elo-game}}

We ran each algorithm on publicly available datasets from a forecasting competition run by FiveThirtyEight\footnote{\url{https://projects.fivethirtyeight.com/2019-nfl-forecasting-game/}} in which users (henceforth called ``forecasters'') make predictions about the outcomes of National Football League (NFL) games. Before each game, FiveThirtyEight releases information on the past performance of the two opposing teams, and forecasters provide probabilistic predictions about which team will win the game. FiveThirtyEight maintains a public leaderboard with the most accurate forecasters, updated after each game. The datasets for the 2018--2019 and 2019--2020 seasons each include all forecasters' predictions, labeled with the forecaster's unique id, information about the corresponding game, and the game's outcome. Each NFL season has a total of 267 games, so in our setting, $T = 267$. For 2018--2019 (respectively, 2019--2020), while 15,702 (15,140) participated, only 302 (375) made predictions for every game. In order to reduce variance, for each value of $K$, we sampled 10 groups of $K$ forecasters from the 302 (respectively, 375), and for each such group, ran each algorithm 50 times.

We evaluate performance using quadratic loss. We compare the cumulative loss of each algorithm against the cumulative loss of the best fixed forecaster in hindsight. For the full information setting, we compare $\wsu$ and $\elfx$ against $\hedge$, which achieves optimal regret guarantees since the quadratic loss is exp-concave, and $\mwu$, which is more similar in form to $\wsu$, in order to evaluate whether anything is lost in terms of regret when incentive compatibility is achieved. For the partial information setting, we compare $\wsux$ against $\texttt{EXP3}$. For each full information algorithm, we run both the variant in which a single expert is selected at each timestep and the variant in which the learner outputs a weighted combination of expert reports (labeled $\texttt{*-Aggr}$). For $\elfxagg$, since $\bpi_t$ cannot be computed in closed form, we approximate it via sampling.

We present the results of our experiments on the 2018--2019 dataset in Figure~\ref{fig:nfl18-19}; the results on the 2019--2020 dataset are in Appendix~\ref{app:1920}, and exhibit similar trends. We note that lines correspond to average regret (across all samples of experts and all repetitions), while the error bands correspond to the 20th and 80th percentiles; this leads to much smaller error bands for larger values of $K$ since the specific sampling of experts has less influence on regret for large $K$. 

Validating our theoretical results, $\wsu$ performs almost identically (in terms of the dependence on both $K$ and $T$) to $\mwu$ when fed the same set of reports---this, of course, does not take into account that $\mwu$ is not incentive compatible and may lead to misreports in practice, potentially degrading predictions. Interestingly, we also see that $\wsuagg$ performs almost identically to $\mwuagg$. This suggests that the performance of $\wsuagg$ is considerably better than the bound in Section~\ref{sec:full-info} implies. It is an interesting open question to see whether better regret guarantees can be proved for $\wsuagg$, perhaps with respect to the best fixed distribution of experts. As expected, both $\wsu$ and $\mwu$ are outperformed by $\hedge$, which achieves optimal regret bounds for squared loss but no incentive guarantees.

$\elfx$ appears to exhibit diminishing regret on this dataset, particularly for $K=5$. However, $\elfx$ and $\elfxagg$ perform worse than $\wsu$ and $\wsuagg$ respectively when fed the same input, particularly when the number of experts is large.  Although $\elfx$ obtains a stronger incentive guarantee, the violations of incentive compatibility for forward-looking experts exhibited by $\wsu$ are very small in our examples. In practice, we expect that $\wsu$ is a superior choice to $\elfx$ when balancing regret and incentive properties, even for forward-looking experts.

For the bandit setting, quite encouragingly, we see that the performance of $\wsux$ is only slightly worse than that of \texttt{EXP3}, and appears significantly better than the $O(T^{2/3})$ regret bound in Section~\ref{sec:bandit} would suggest. This could be a byproduct of our analysis not being tight, and it remains an open question whether this bound can be improved.

The experiments presented in this section focus on settings with relatively small horizons $T$ since an NFL season has only $267$ matches. In Appendix~\ref{app:monte-carlo}, we present our results (which also validate our theoretical analysis) for Monte Carlo simulations for larger horizons.

%% file: 7-conclusion.tex
\section{Conclusion and Open Questions}
\label{sec:future}

We studied the problem of online learning with strategic experts. We introduced algorithms that are simultaneously no-regret and incentive-compatible, and assessed their performance experimentally on data from FiveThirtyEight. Several open questions arise. In the full-information setting, there is the question of whether an incentive-compatible algorithm exists with better regret bounds for the special case of exp-concave bounded proper loss functions. For the bandit setting, there is the question of whether there exist incentive-compatible algorithms that bridge the gap between the regret of $\wsux$ and that of $\texttt{EXP3}$, and whether a better regret guarantee could be proved for $\wsux$ via a tighter analysis. There is also the question of whether $\elfx$ is indeed no-regret, as our experimental results might suggest. More broadly, the most important research question that we believe needs to be addressed in online learning from strategic agents is the quantification of the tradeoff between incentive-incompatibility and standard learning guarantees and how to balance these in practice.

%% file: acks.tex
\section*{Acknowledgments}

The authors are grateful to Yiling Chen, Ariel Procaccia, Rob Schapire, Vasilis Syrgkanis, and Jens Witkowski for helpful discussions at different stages of this work, and to the anonymous reviewers for their comments and suggestions.

%% file: appendix.tex
\newpage
\appendix

\section{Gradient Descent Violates Incentive Compatibility}
\label{sec:gd-counterexample}

In gradient descent the loss function that we are trying to optimize is $(r_t - \sum_{i \in [K]} \pi_{i,t} p_{i,t})^2$. Assume that for all the experts $j \neq i$, $b_{j,t} = p_{j,t} = 0$. Then, from the perspective of expert $i$ and according to their belief $b_{i,t}$ their expected weight at the next round is 
\[ \mathbb{E}_{r_t \sim \Bern(b_{i,t})} \left[ \pi_{i,t+1} \right] = \underbrace{b_{i,t} \cdot \frac{\pi_{i,t} + 2 \eta p_{i,t}(1 - \pi_{i,t}p_{i,t})}{1+2\eta p_{i,t} (1-\pi_{i,t}p_{i,t})}}_{Q_1} + \underbrace{(1-b_{i,t}) \cdot \frac{\pi_{i,t} - 2\eta p_{i,t}^2 \pi_{i,t}}{1 - 2\eta p_{i,t}^2 \pi_{i,t}}}_{Q_2}. \] 
We begin with a specific case: $K = 10, \pi_{i,t} = 0.1, b_{i,t}=0.6$. Then, for any $\eta \ge 2.85 \cdot 10^{-15}$ reporting $p_{i,t} = 0.61$ is a beneficial manipulation for the expert. To construct similar counterexamples for any $\eta$, one needs to focus on cases where $\pi_{i,t} \to 0$ (which can be achieved by, for instance, allowing the number of experts to grow large), hence $Q_2$ is almost $0$ and $Q_1$ ends up thus being maximized when $p_{i,t}$ is maximum (i.e., for $p_{i,t} \to 1$).

\section{Supplementary Material for Section~\ref{sec:full-info}}

\subsection{Proof of the Validity of $\wsu$}

\begin{lemma}\label{lem:prop-distr-wsu}
The weights $\bpi_{t}$ produced by $\wsu$ form a well-defined probability distributions for all $t \in [T+1]$.
\end{lemma}

\begin{proof}
To show that a distribution is valid, we must show that the components are non-negative and sum to one. We do this inductively. The base case is satisfied trivially since $\pi_{i,1} = 1/K$ for all $i$. Now assume that $\bpi_t$ is a valid probability distribution.
For $t+1$, from Equation~\ref{eqn:wswm-update}, we have
$
\pi_{i,t+1} \geq \eta \Gamma^{\WSWM}_i(\vp_t, \bpi_t, r_t) 
            \geq 0
$
where the last inequality follows from the properties of $\WSWM$ and the assumption that $\bpi_t$ is a valid distribution.
We also have
\begin{align*}
	\sum_{i \in [K]} \pi_{i,t+1} &= \eta \sum_{i \in [K]} \Gamma^{\WSWM}_i (\vp_t,\bpi_t, r_t) + (1-\eta) \sum_{i \in [K]} \pi_{i,t} 
	= \eta \sum_{i \in [K]} \pi_{i,t} + (1-\eta) \sum_{i \in [K]} \pi_{i,t} =1
\end{align*}
where the second equality follows from the fact that $\WSWM$ is budget balanced and the final equality from the assumption that $\bpi_t$ is a valid distribution.
\end{proof}

\subsection{Technical Lemma}

\begin{lemma}\label{lem:technical}
For all $x \leq 1/2$, it holds that: $\ln (1 -x) \geq -x -x^2$.
\end{lemma}

\begin{proof}
Let function $f(x), x \leq 1/2$ be defined as $f(x) = \ln (1-x) + x + x^2$.  It suffices to show that $f(x) \geq 0$ for the domain of interest. Taking the first derivative we get
 $$f'(x) = \frac{-x(2x-1)}{1-x}.$$
For $x\leq 1/2$, $f'(x) = 0$ for $x = 0$ and $x = 1/2$. Now, since $f'(x) \leq 0, x \leq 0$ and $f'(x) \geq 0, 0 \leq x \leq 1/2$ we get that $f(x)$ is decreasing for $x \in (-\infty, 0]$ and increasing for $x \in [0, 1/2]$. As such, it presents a minimum at $x = 0$, and for $x \leq 1/2$, $f(x) \geq f(0) = \ln (1) + 0 + 0 = 0.$ Hence, $\ln (1-x) \geq -x -x^2$. 
\end{proof}

\subsection{Regret of $\wsu$ for Unknown Time Horizon $T$}\label{app:anytime-wsu}

In order to provide an anytime variant of $\wsu$, we use a standard doubling trick \citep{ACBFS02}. We maintain an estimated upper bound on the time horizon $T$, denoted $n$, starting with $n=1$. For all rounds $t \in (n/2, n]$, we run $\wsu$ using $\eta = \eta_n = \sqrt{\ln (K)/n}$. If at any round $t'$ we have that $t' > n$, then we double our estimated horizon upper bound to $2n$ (changing $\eta$ accordingly) and restart $\wsu$ by initializing all weights to $1/K$. As we prove below, this process increases the regret only by constants.

\begin{lemma}
For an a-priori unknown time horizon $T$, $\wsu$ with a doubling trick is incentive-compatible and incurs regret $R \leq \frac{2\sqrt{2}}{\sqrt{2} - 1} \sqrt{T \ln K}$.
\end{lemma}

\begin{proof}
Using the doubling trick, the time horizon $T$ can be divided into phases during which $n$, and hence also $\eta$, remain constant. Because of this, from the perspective of an expert $i$, it does not matter in which phase the algorithm is currently at: their probability at the next round is computed as $\pi_{i,t+1} = \eta_n \Gamma_i^{\texttt{WSWM}}(\vp_t, \bpi_t,r_t) + (1 - \eta_n) \pi_{i,t}$, hence it still is a convex combination of a $\texttt{WSWM}$ payment and $\pi_{i,t}$, which cannot be influenced by $i$'s report at round $t$. Since the algorithm every time restarts (i.e., experts' weights are re-initialized to $1/K$) using the new $\eta_n$ for all the rounds, this ends up being equivalent to having a constant $\eta$ throughout $T$ timesteps in terms of incentives. 

Since the length of each phase, $n$, is doubled at the end of each phase, the number of these phases is at most $\lceil \log T \rceil$. Additionally, the actual regret throughout the $T$ rounds is upper-bounded by the sum of the regret of each phase. Hence, using Theorem~\ref{thm:no-regr-myopic2} we have that: 
\begin{align*}
R &\leq \sum_{n=0}^{\lfloor \log T \rfloor} 2 \sqrt{2^n \ln K} \leq \left(2 \sqrt{\ln K}\right)\sum_{n=0}^{\lfloor \log T \rfloor} \left(\sqrt{2}\right)^n \\ 
&= \left(2 \sqrt{\ln K}\right)\frac{1 - \sqrt{2}^{\lfloor \log T \rfloor + 1}}{1 - \sqrt{2}} \\  
&= \left(2 \sqrt{\ln K}\right)\frac{2^{\frac{1}{2} \lfloor \log T \rfloor} \cdot \sqrt{2} - 1}{\sqrt{2} -1}\\
&\leq \left(2 \sqrt{\ln K}\right)\frac{2^{\left \lfloor \log T^{1/2} \right \rfloor} \cdot \sqrt{2}}{\sqrt{2} -1}\\
&= \left(2 \sqrt{2} \sqrt{\ln K}\right)\frac{T^{1/2}}{\sqrt{2} -1}= \frac{\left(2 \sqrt{2} \sqrt{T \ln K}\right)}{\sqrt{2}-1}
\end{align*}
where the first equality comes from the definition of a geometric series with rate $\sqrt{2}$.
This concludes our proof. 
\end{proof}

\section{Supplementary Material for Section~\ref{sec:bandit}}\label{app:bandit-known-T}

\subsection{Proof of Lemma~\ref{lem:unbiased}}
For the first moment, we have: 
$$\E_{I_t \sim \tpi_t} \left[ \hell_{i,t} \right] = \sum_{j \in [K]} \tpi_{j,t} \frac{\ell_{i,t}\1 \{ j \!=\! i\}}{\tpi_{i,t}} = \ell_{i,t}.$$ 
For the second moment, we have: 
$$\E_{I_t \sim \tpi_t} \left[ \hell_{i,t}^2 \right] = \sum_{j \in [K]} \tpi_{j,t} \frac{\ell_{i,t}^2 \1 \{i = j\}}{\tpi_{i,t}^2} = \frac{\ell_{i,t}^2}{\tpi_{i,t}} \leq \frac{1}{\tpi_{i,t}},$$
where the last inequality uses the fact that $\ell_{i,t} \in [0,1], \forall i \in [K], \forall t \in [T]$.
\qed

\subsection{Regret of $\wsux$ for Unknown Time Horizon $T$}\label{app:anytime-wsux}

Similarly to Appendix~\ref{app:anytime-wsu}, in this subsection we use the doubling trick \citep{ACBFS02} in order to achieve regret guarantees for $\wsux$ for the case of an unknown horizon $T$. Formally, we prove the following.

\begin{lemma}
For an a-priori unknown time horizon $T$, $\wsux$ with a doubling trick is incentive-compatible and incurs regret $R \leq \frac{8}{2^{2/3} - 1} T^{2/3} (K \ln K)^{1/3}$.
\end{lemma}

\begin{proof}
Algorithm $\wsux$ is divided into phases during which $n$ and $\eta$ remain constant. This coupled with the fact that at every phase the algorithm is restarted and the experts' weights are re-initialized to $1/K$ (i.e., hence all previous weights have been updated with the same $\eta$) means that from the perspective of an expect, the incentives structure remains the same. As a result, $\wsux$ with a doubling trick is incentive-compatible.

The number of the algorithm's phases is at most $\lfloor \log T \rfloor$. The actual regret throughout the $T$ rounds is upper bounded by the sum of the regret of each phase. So, from Theorem~\ref{thm:regr-bandit} we obtain that: 
\begin{align*}
R   &\leq \sum_{n=0}^{\lfloor \log T \rfloor} 2 \cdot 4^{2/3} \cdot (K \ln K)^{1/3} \cdot \left(2^n\right)^{2/3} = 2 \cdot 4^{2/3} \cdot (K \ln K)^{1/3} \sum_{n=0}^{\lfloor \log T \rfloor} \left( 2^{2/3} \right)^n \\
    &=  2 \cdot 4^{2/3} \cdot (K \ln K)^{1/3} \frac{1 - \left(2^{2/3}\right)^{\lfloor \log T \rfloor + 1}}{1 - 2^{2/3}} \leq 2 \cdot 2^{2/3} \cdot 4^{2/3}\cdot (K \ln K)^{1/3}\frac{\left(2^{2/3}\right)^{\lfloor \log T \rfloor}}{2^{2/3}-1} \\ 
    &= 2 \cdot 2^{2/3} \cdot 4^{2/3}\cdot (K \ln K)^{1/3}\frac{\left(2\right)^{\frac{2}{3}\lfloor \log T \rfloor}}{2^{2/3}-1} =  \frac{8}{2^{2/3} - 1}(K \ln K)^{1/3}T^{2/3}
\end{align*}
This concludes our proof.
\end{proof}

\input{appendix-elf}

\section{Supplementary Material for Section~\ref{sec:experiments}.}

\subsection{FiveThirtyEight NFL 2019--2020 Dataset}\label{app:1920}

In this subsection we present in Figure~\ref{fig:nfl19-20} the results of our experiments for the 2019--2020 FiveThirtyEight NFL dataset. The findings and conclusions are almost identical to those drawn using the 2018--2019 FiveThirtyEight NFL dataset found in Section~\ref{sec:experiments}.

\subsection{Monte Carlo Simulations with Large Horizon $T$}\label{app:monte-carlo}

In this subsection, we present our results for Monte Carlo simulations for larger horizons in Figure~\ref{fig:simulations}. We simulated the following setup: $K = 50$, $T = 2500$ and we repeated the simulations for $50$ repetitions. The lines correspond to average regret (across all repetitions), and the error bands in Figure~\ref{fig:simulations} correspond to the 20th and the 80th percentiles. 

The realized outcomes are sampled as follows: for rounds $0 \leq t \leq T/2$, $r_t \sim \Bern(0.4)$, and for rounds $T/2 + 1 \leq t \leq T$, $r_t \sim \Bern(0.6)$. 
The $K$ experts are randomly partitioned into three equal-sized groups sampling their beliefs from three different distributions: for experts in the first group we draw $b_{i,t} \sim \texttt{Unif}[0,0.7]$ for all rounds $t$, for the second group $b_{i,t} \sim \texttt{Unif}[0.3,1]$ for all $t$, and for the third group $b_{i,t} \sim \texttt{Unif}[0,1]$ for all $t$. 
As a result, in expectation, experts from the first group perform best for the first $T/2$ rounds, the second group performs best for the next $T/2$ rounds, while the third group performs best when all $T$ rounds are considered. 

\begin{figure*}[htbp]
\centering
\subfigure{\includegraphics[width=0.3\textwidth]{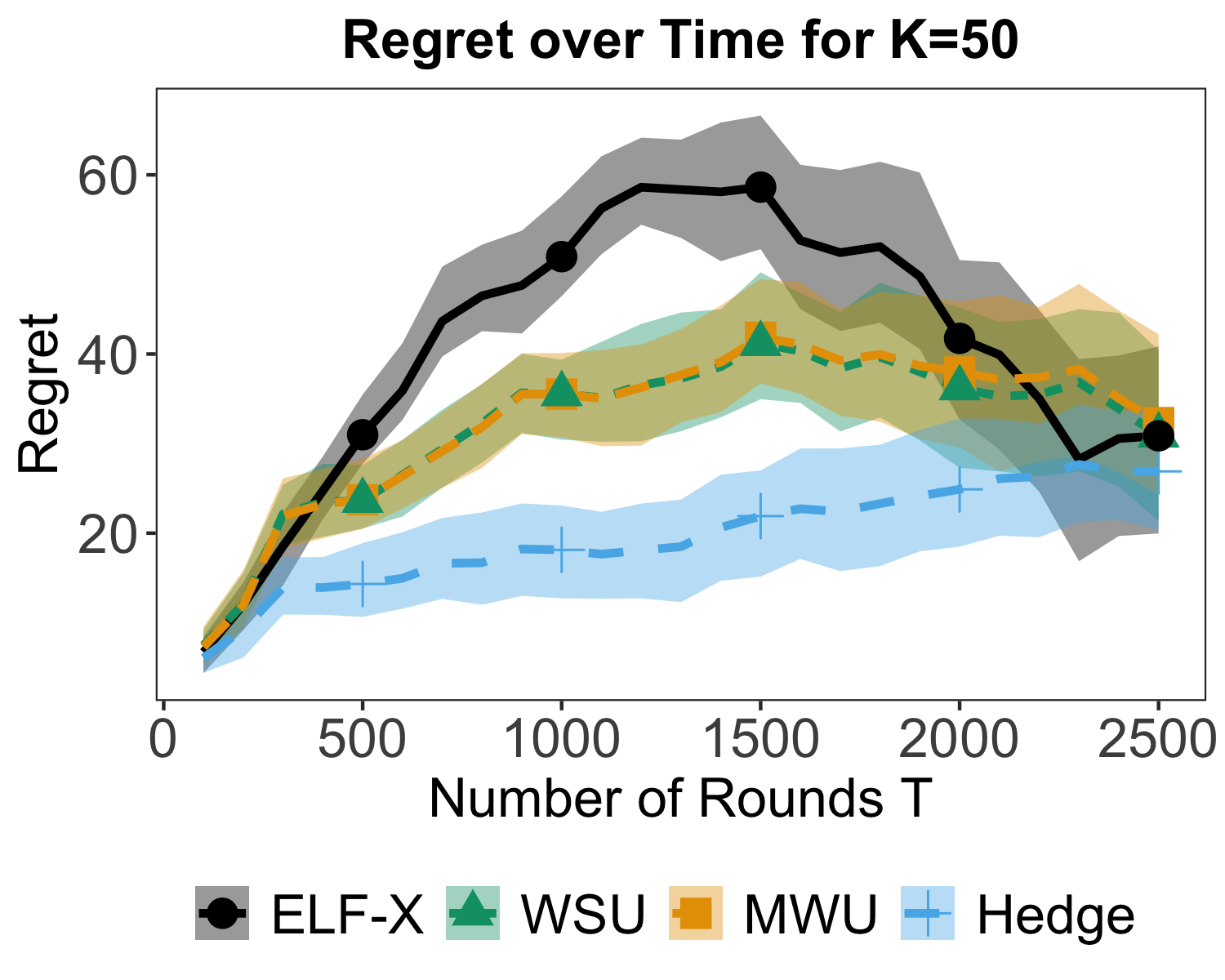}} \hfill
\subfigure{\includegraphics[width=0.3\textwidth]{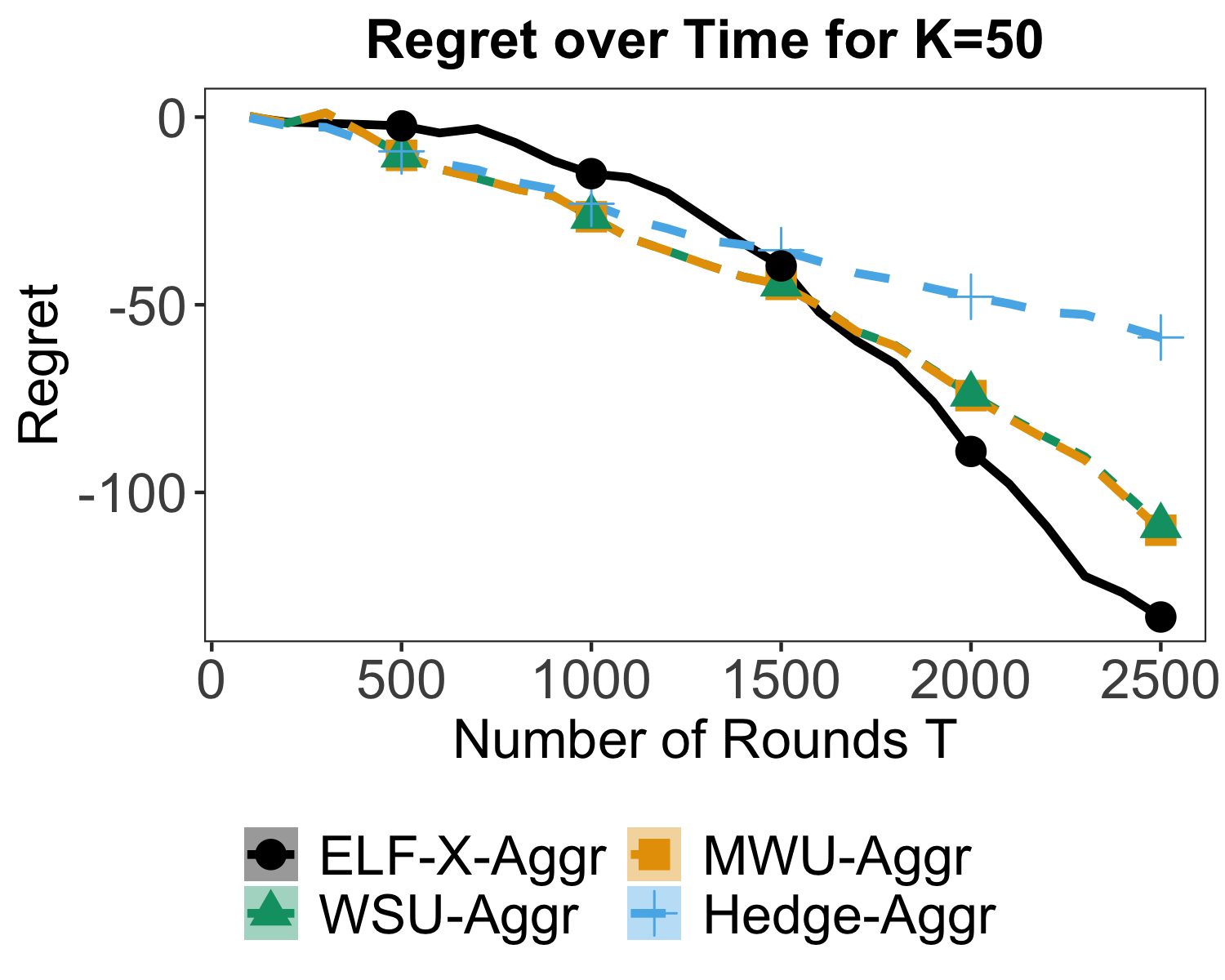}} \hfill
\subfigure{\includegraphics[width=0.3\textwidth]{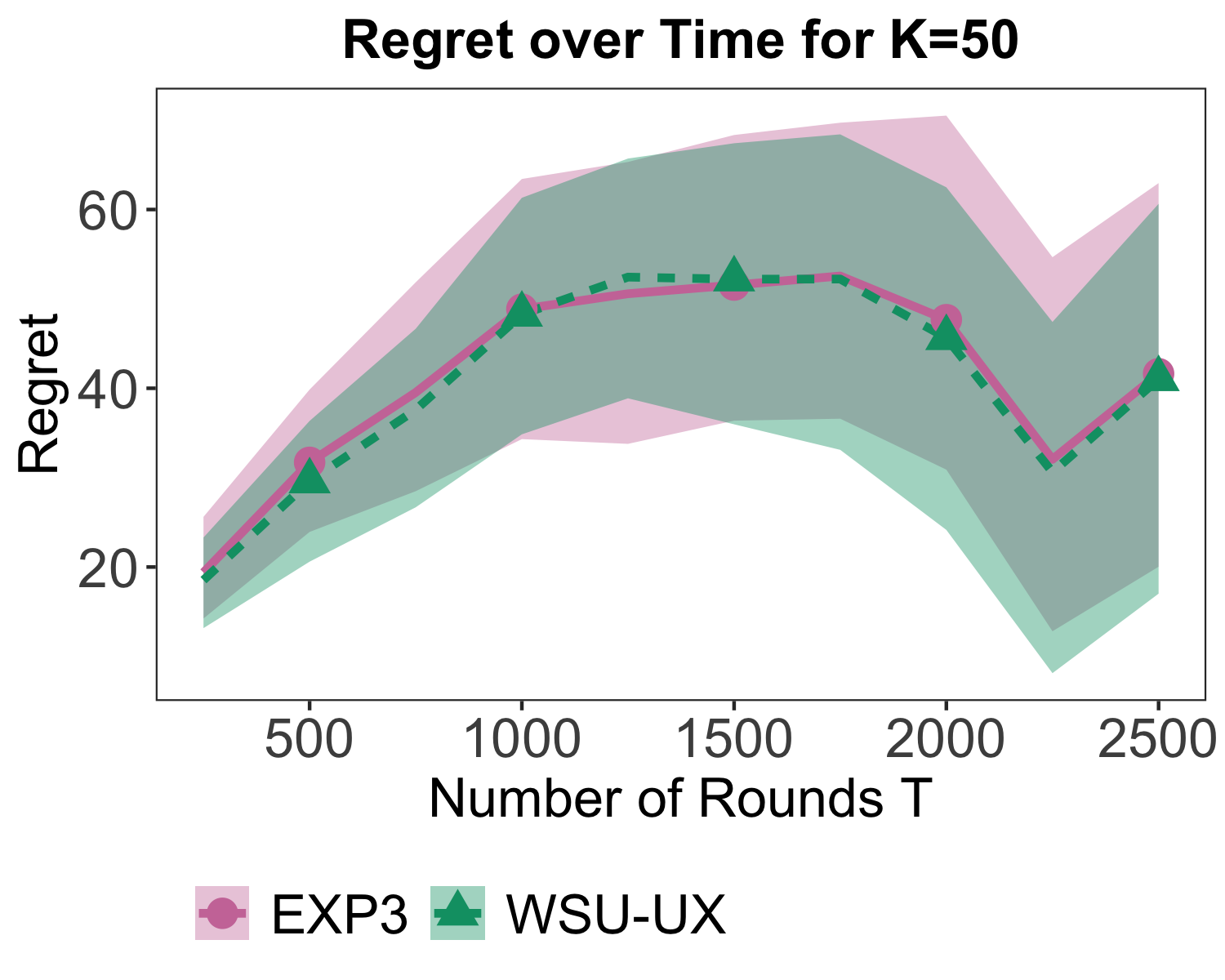}} \hfill
\caption{Simulation Results for $K=50$ experts. Left: Full-information setting with $\bp_t$ the prediction of a single expert chosen according to $\bpi_t$. Middle: Full-information setting with $\bp_t=\sum_{i \in [K]} \pi_{i,t}p_{i,t}$. Right: Partial information setting.}
\label{fig:simulations}
\end{figure*}%

Due to the way we constructed the simulation parameters, examining the performance of the algorithms for timesteps between $[0,T/2]$ provides intuition about their performance for settings where the experts' performance is relatively stable over time. However, their performance for timesteps between $[T/2, T]$ provides intuition for settings where the best expert is shifting over time. As a result, for timesteps between $[0,T/2]$ our findings are similar to the findings of our experiments on the FiveThirtyEight NFL datasets: $\elf$ performs worse than $\wsu$ (which performs identically to $\mwu$) and worse than $\hedge$, and $\wsux$ performs almost identically to $\texttt{EXP3}$ despite our weaker theoretical bound.  

Interestingly, for timesteps between $[T/2, T]$ we find that $\elf$ briefly performs better than $\wsu,\mwu$ and $\hedge$. We conjecture that this is due to the fact that $\elf$ in the first $T/2$ timesteps takes longer than $\mwu,\wsu$ and $\hedge$ to converge to experts in the first group. Because these experts are no longer optimal throughout the $T$ timesteps, $\elf$ has an advantage over the other algorithms. 

Lastly, we note that the regret performance of the aggregating variants of all algorithms is always negative due to the fact that the expectation over all experts is very close to issuing the optimal prediction for all rounds. As a result, a prediction that takes into account all of their predictions in a weighted fashion performs much better than the prediction of any fixed expert in hindsight. We also note that the fact that $\hedge$ is performing worse than the other algorithms is not contradicting the theoretical results, which are only stated in terms of \emph{worst case} upper bounds. Finally, we see that $\elfxagg$ performs better than \emph{all} algorithms in this setting. Explaining this phenomenon theoretically even for particular settings is a question of great interest.

%% file: appendix-elf.tex
\section{Supplementary Material for Section~\ref{sec:elf}}

We begin with a definition of incentive compatibility when experts may look more than one timestep into the future. This stronger version of incentive compatibility requires that for any timestep $t$ and future timestep $t^f>t$, experts maximize their expected weight at timestep $t^f$ by truthfully reporting their beliefs at all timesteps between $t$ and $t^f$.

\begin{definition}[Incentive Compatibility for Forward-Looking Experts]\label{def:strong-ic}
An online learning algorithm is \emph{incentive-compatible for forward-looking experts} if for every timestep $t \in [T]$ and every future timestep $t^f >t$, every expert $i$ with beliefs $(b_{i,t'})_{t \le t'<t^f}$, and every set of reports of expert $i$, $(p_{i,t'})_{t \le t'<t^f}$, reports of the other experts $(\vp_{-i,t'})_{t \le t'<t^f}$, and every history of reports $(\vp_{t''})_{t''<t}$ and outcomes $(r_{t''})_{t''<t}$,
\begin{align*}
&\E_{(r_{t'} \sim \Bern(b_{i,t'}))_{t \le t'<t^f}} [ \pi_{i,t^f} |  \left(b_{i,t'} \right)_{t \le t'<t^f} , \left( \vp_{-i,t'} \right)_{t \le t'<t^f},(\vp_{t''})_{t''<t}, (r_{t''})_{t''<t} ] \\
&\geq \E_{(r_{t'} \sim \Bern(b_{i,t'}))_{t \le t'<t^f}} [ \pi_{i,t^f} |  \left( p_{i,t'} \right)_{t \le t'<t^f} , \left( \vp_{-i,t'} \right)_{t \le t'<t^f}, (\vp_{t''})_{t''<t}, (r_{t''})_{t''<t}] .
\end{align*}
\end{definition}

$\wsu$ and $\wsux$ do not satisfy incentive compatibility for forward-looking experts. We present an example for $\wsu$, but note that adding a small amount of uniform exploration will still yield a violation. Observe also that the incentives to deviate in the following example are very small. It is an open problem whether $\wsu$ can sometimes produce larger incentives to misreport, or, conversely, whether it satisfies some notion of $\epsilon$-incentive compatibility.

\begin{theorem}
	$\wsu$ is not incentive-compatible for forward-looking experts.
\end{theorem}

\begin{proof}
	Let $K=2$, $T=3$, and $b_{1,1}=0.7$, $b_{1,2}=0.6$, $b_{2,1}=0.4$, and $b_{2,2}=0$. If both experts report truthfully at both rounds, it can be checked that the expected weight of expert 1 at timestep 3 is $\E_{r_1 \sim \Bern(b_{1,1}), r_2 \sim \Bern(b_{1,2})}[\pi_{1,3}] = 0.5+0.1125\eta -0.00188325\eta^2$. However, if expert one instead reports $p_{1,1}=0.699$, then his expected weight at timestep 3 is $\E_{r_1 \sim \Bern(b_{1,1}), r_2 \sim \Bern(b_{1,2})}[\pi_{1,3}]=0.5+0.112499944\eta-0.0018719238\eta^3$.
	It is easy to check that the latter is larger than the former for all $\eta>0.0703$.

	For ease of presentation we do not present a possible manipulation for smaller values of $\eta$, but note that such manipulations can be obtained by considering $0.699<p_{1,1}<0.7$.
\end{proof}

For completeness, we include here some discussion as to the distinction between our $\elfx$ algorithm and the $\elf$ algorithm of \citet{WFVPK18}, who designed $\elf$ for selecting the winner of a forecasting competition.

$\elf$ works similarly to $\elfx$ as defined in Section~\ref{sec:elf}, except that the ``winner'' $x_{\tau}$ of each round $\tau \in [t]$ is chosen with probability $\frac{1}{K}\left( 1 - \ell_{i,t'} + \frac{1}{K-1}\sum_{j \in [K] \setminus \{ i \} } \ell_{j,t'}  \right)$.

Unfortunately, direct application of $\elf$ in the online learning settings we are considering in this paper yields an algorithm with linear regret in the worst case. In particular, when there are two experts and the reports of each expert are always either 0 or 1, $\elf$ reduces to the Follow-the-Leader algorithm that, at every timestep, selects the expert with the lowest cumulative loss. It is well known that Follow-the-Leader has linear regret even under this restriction. $\elfx$ avoids this problem by adding additional randomness into the selection of each round's winner.

\begin{figure*}[t!]
\centering
\subfigure{\includegraphics[width=0.3\textwidth]{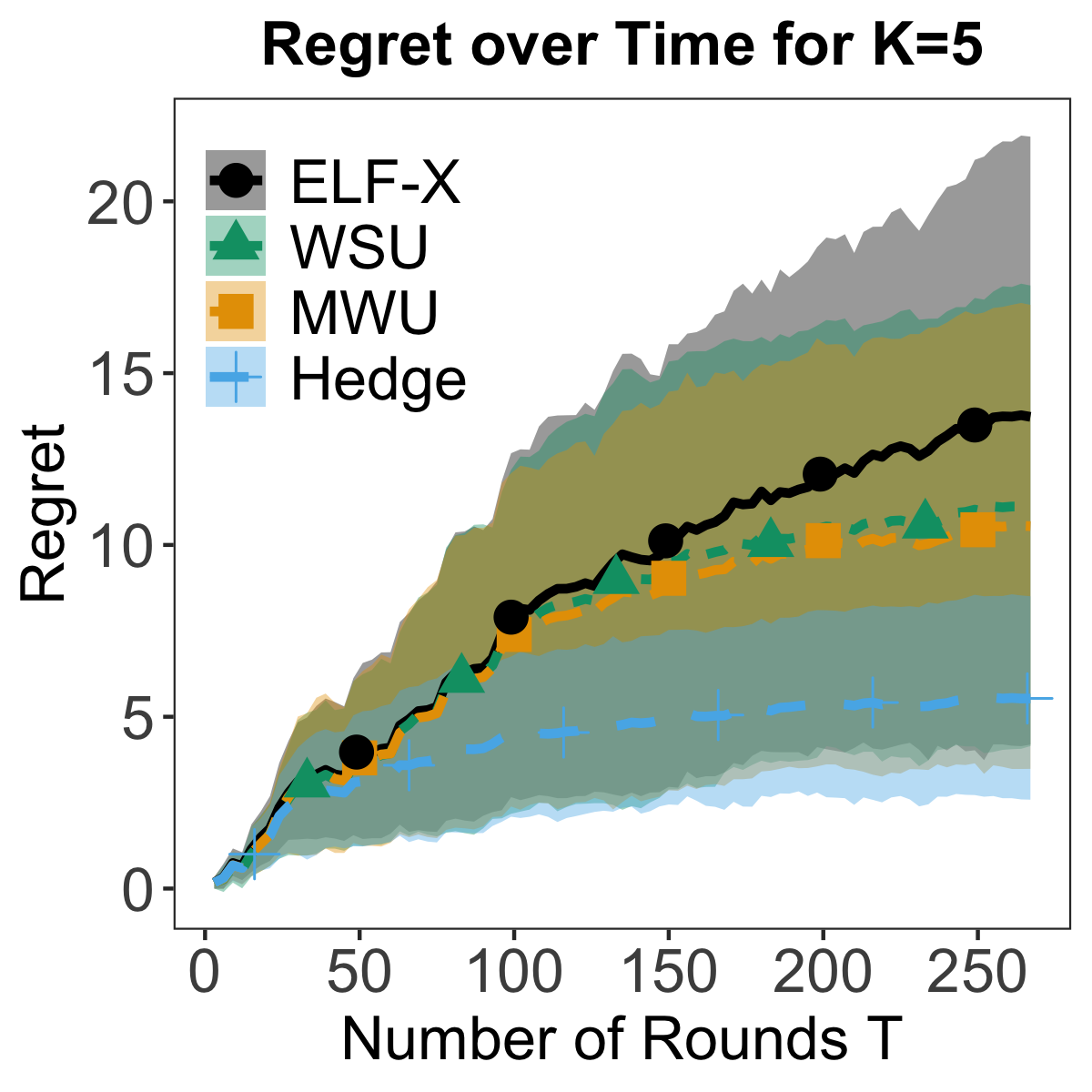}} \hfill
\subfigure{\includegraphics[width=0.3\textwidth]{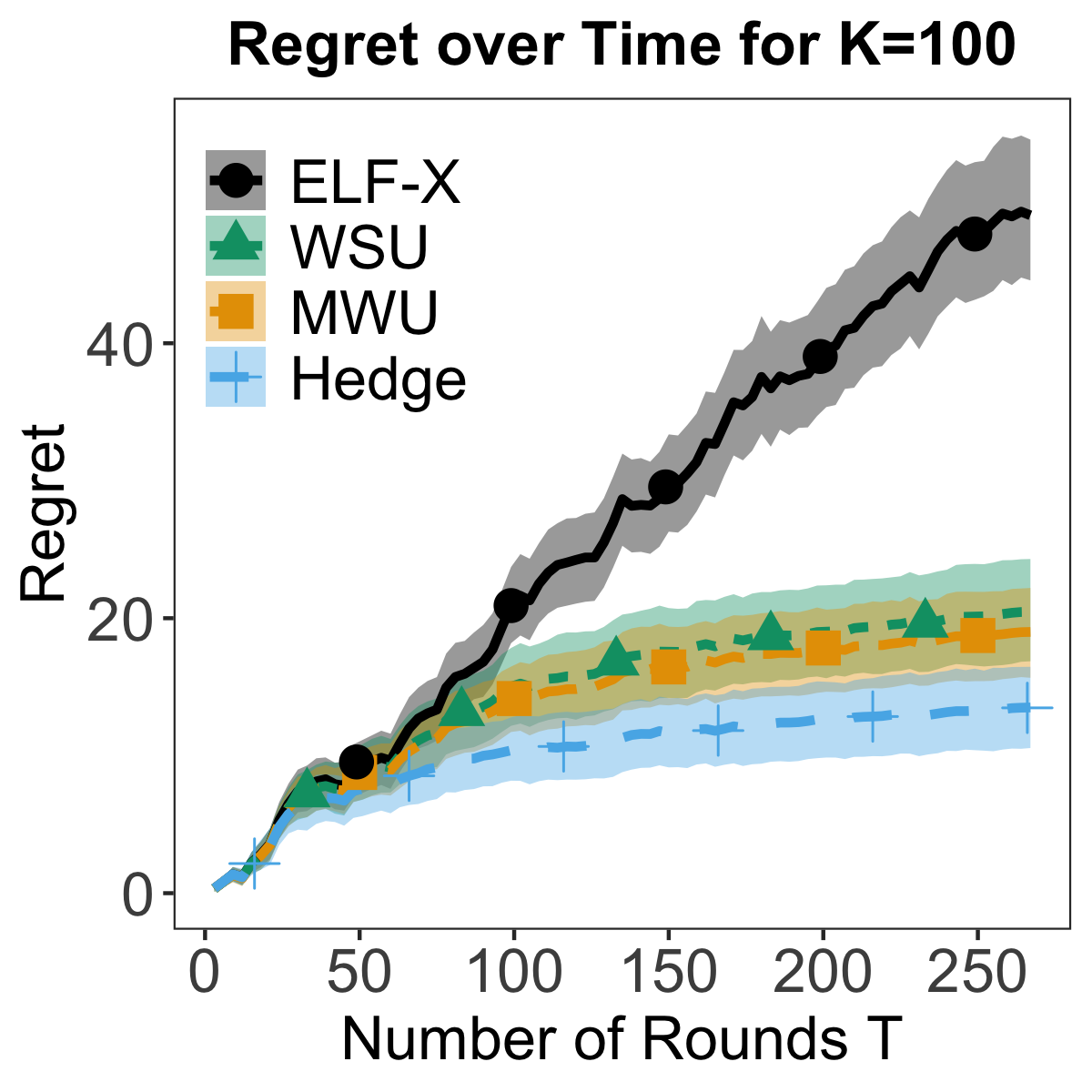}} \hfill
\subfigure{\includegraphics[width=0.3\textwidth]{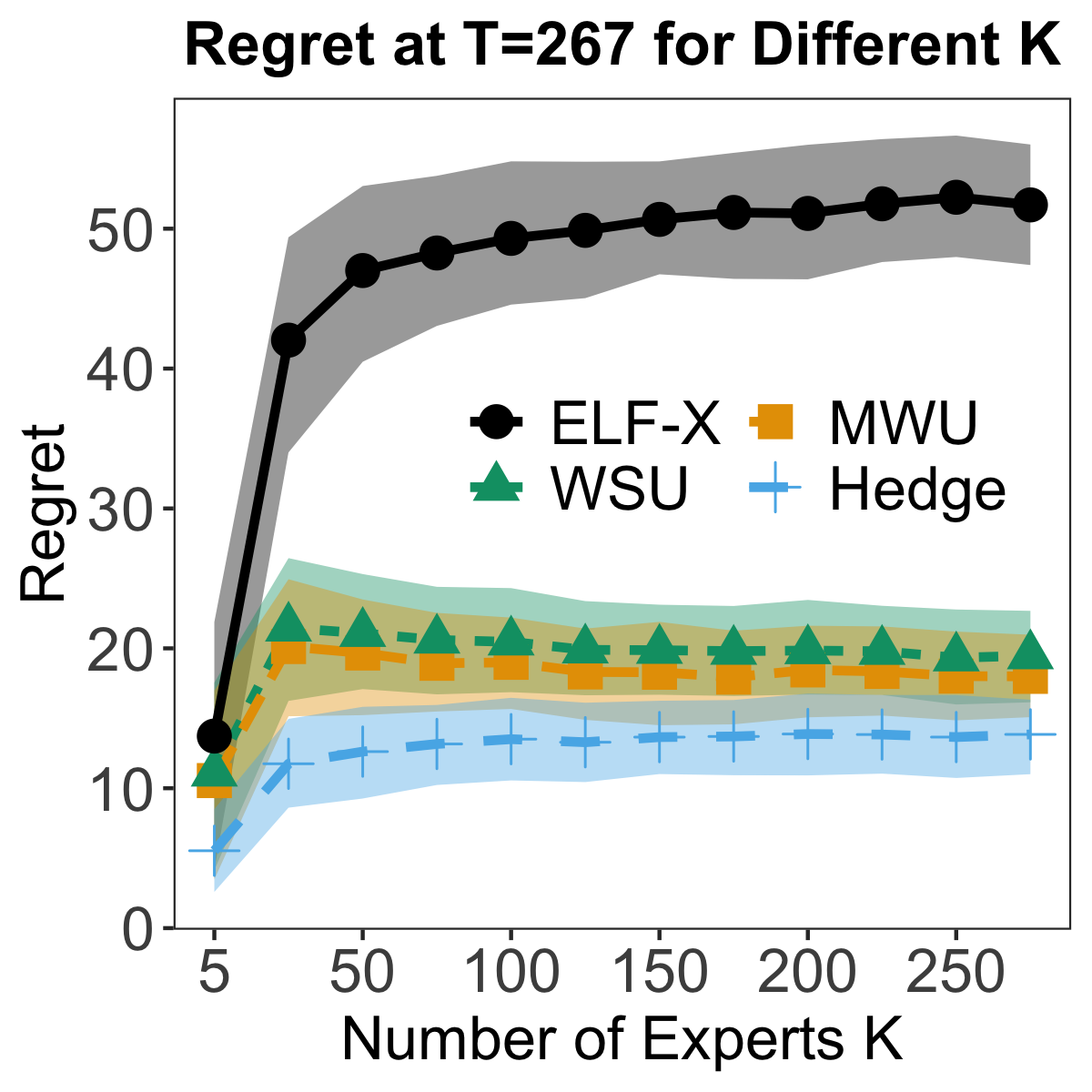}} \\  
\subfigure{\includegraphics[width=0.3\textwidth]{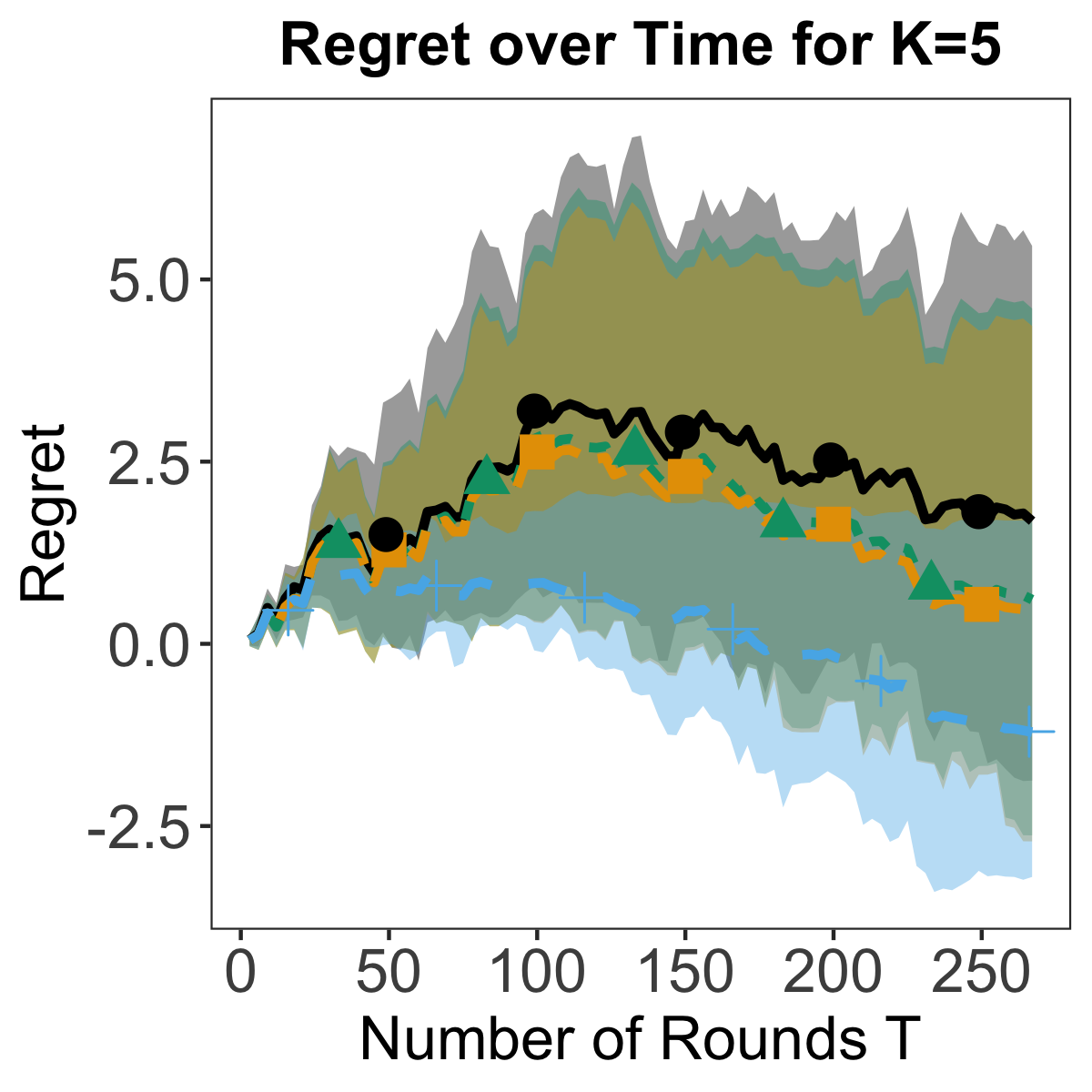}} \hfill
\subfigure{\includegraphics[width=0.3\textwidth]{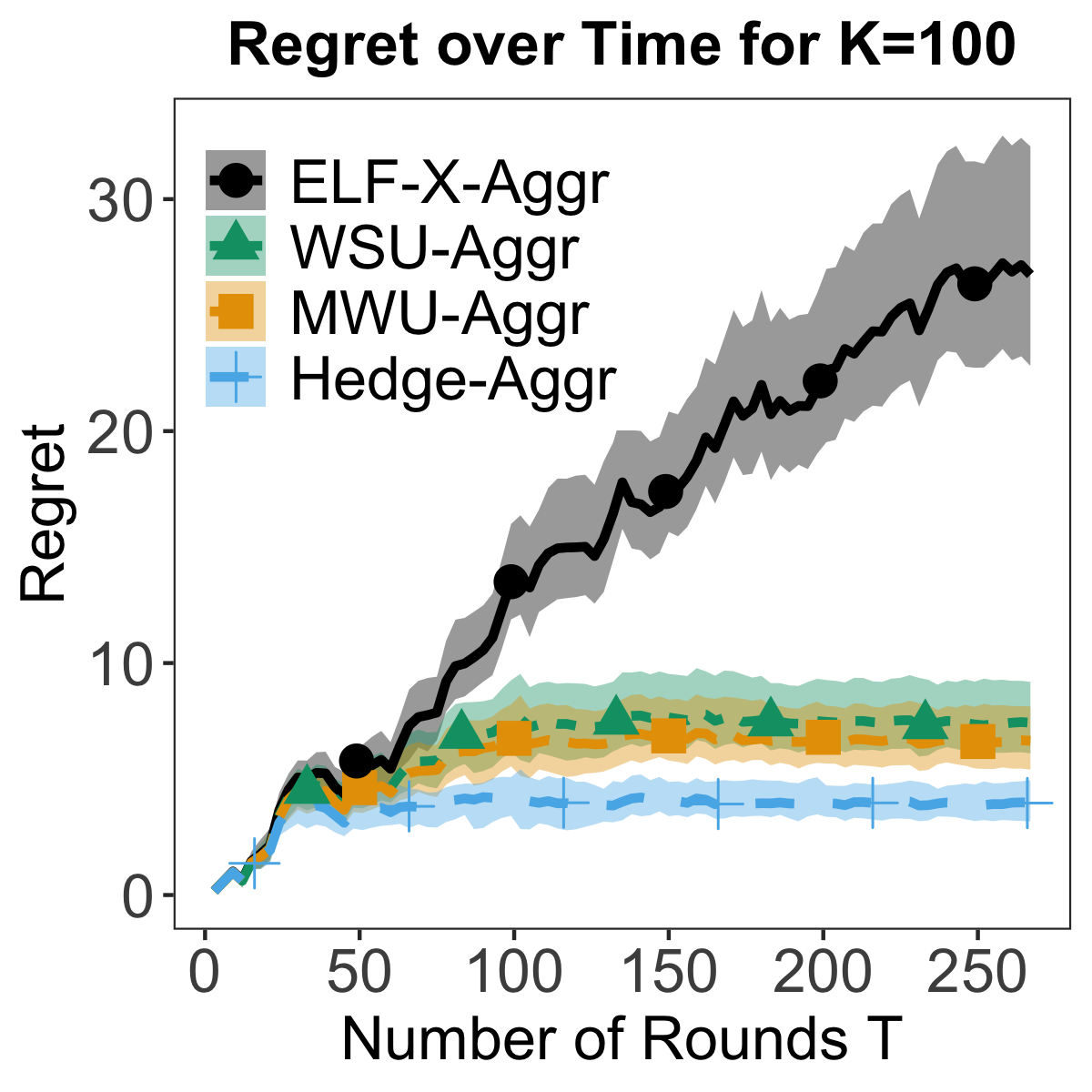}} \hfill
\subfigure{\includegraphics[width=0.3\textwidth]{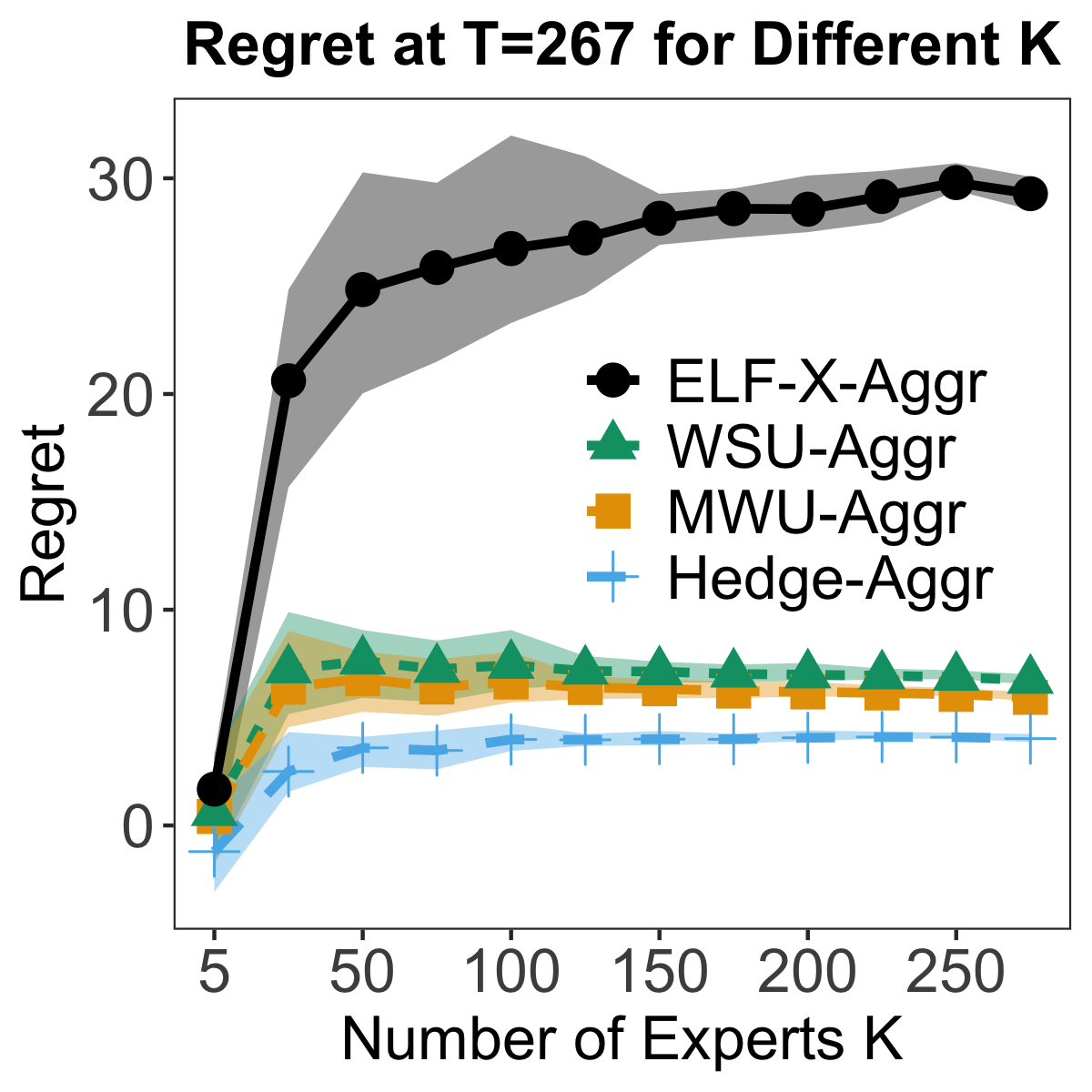}} \\ 
\subfigure{\includegraphics[width=0.3\textwidth]{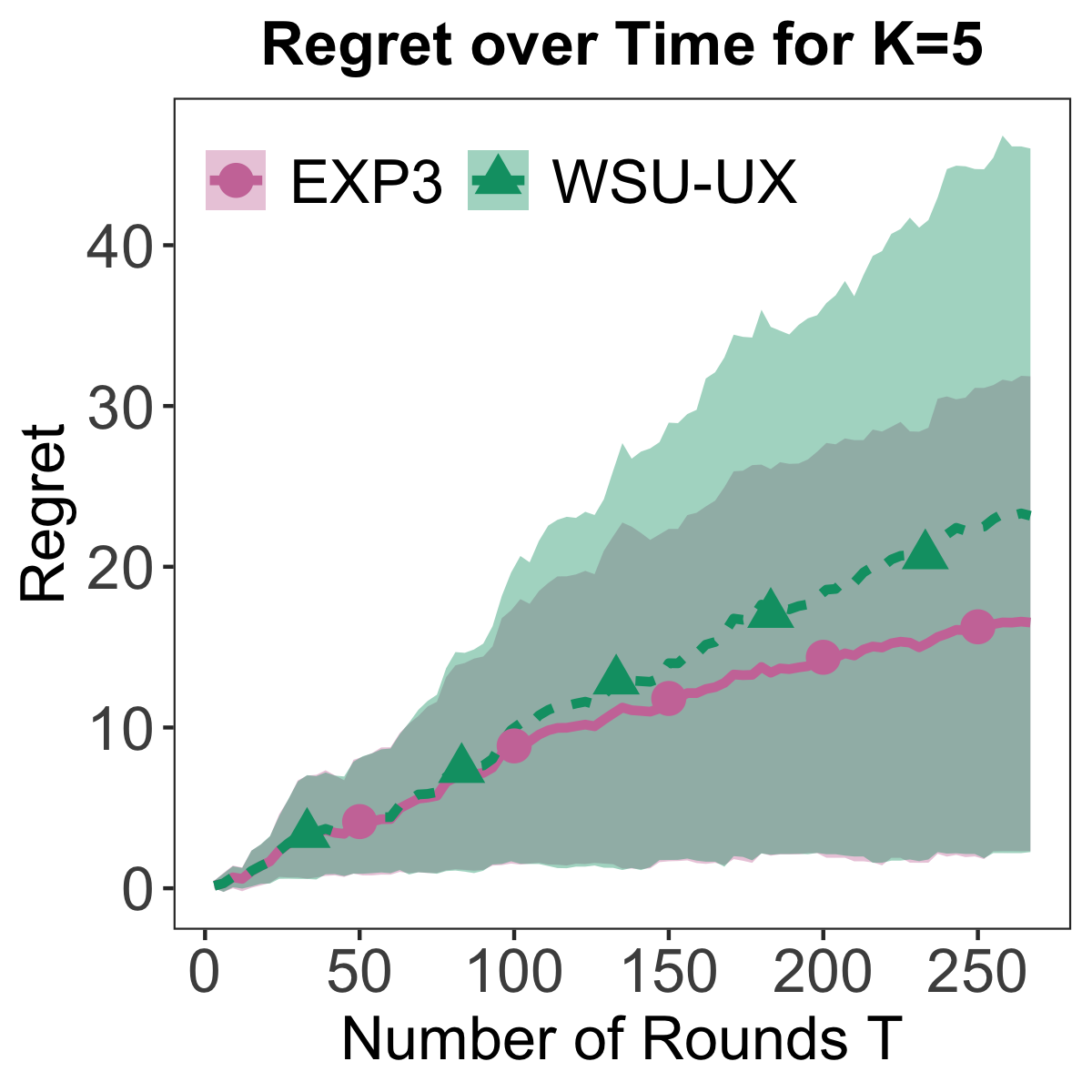}} \hfill
\subfigure{\includegraphics[width=0.3\textwidth]{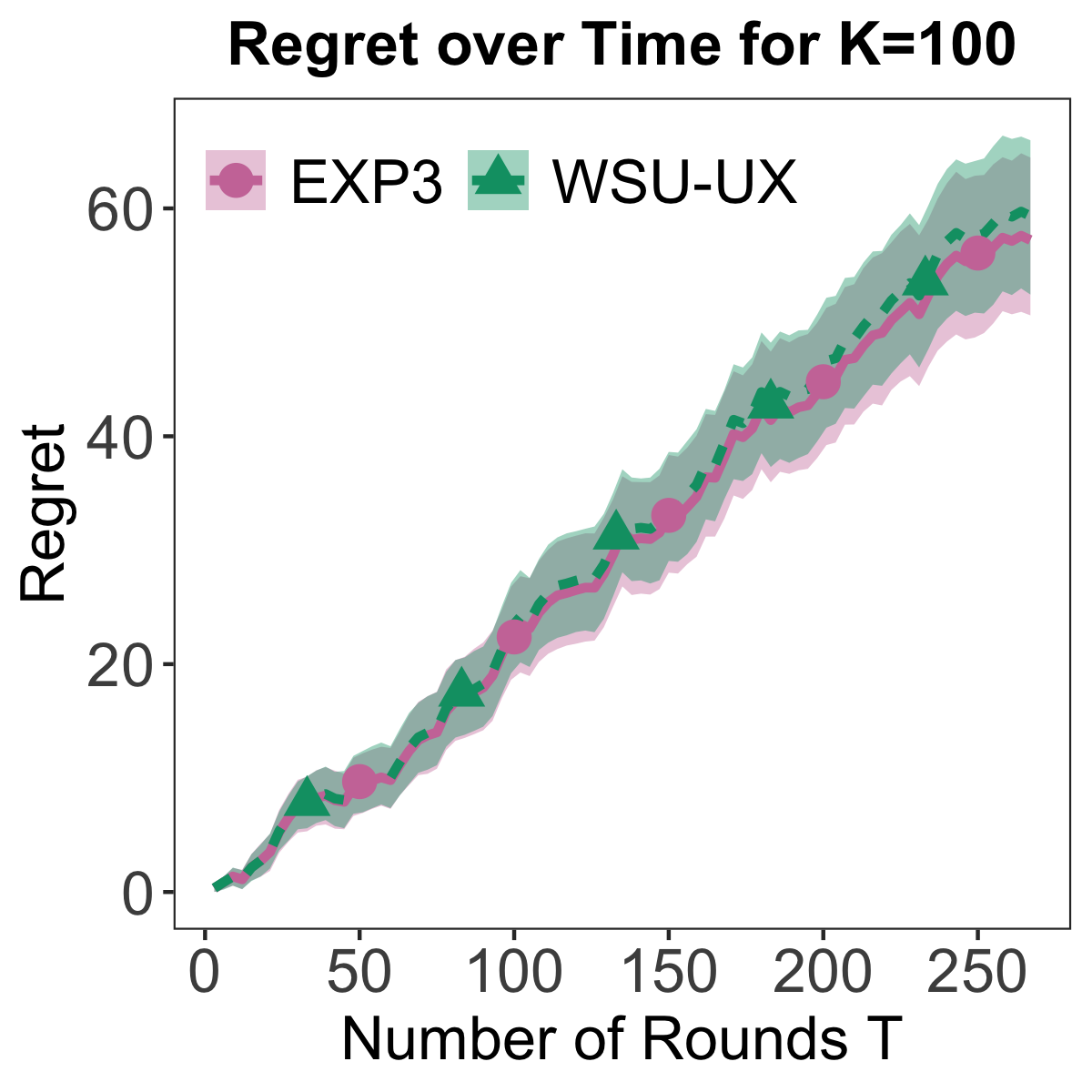}} \hfill
\subfigure{\includegraphics[width=0.3\textwidth]{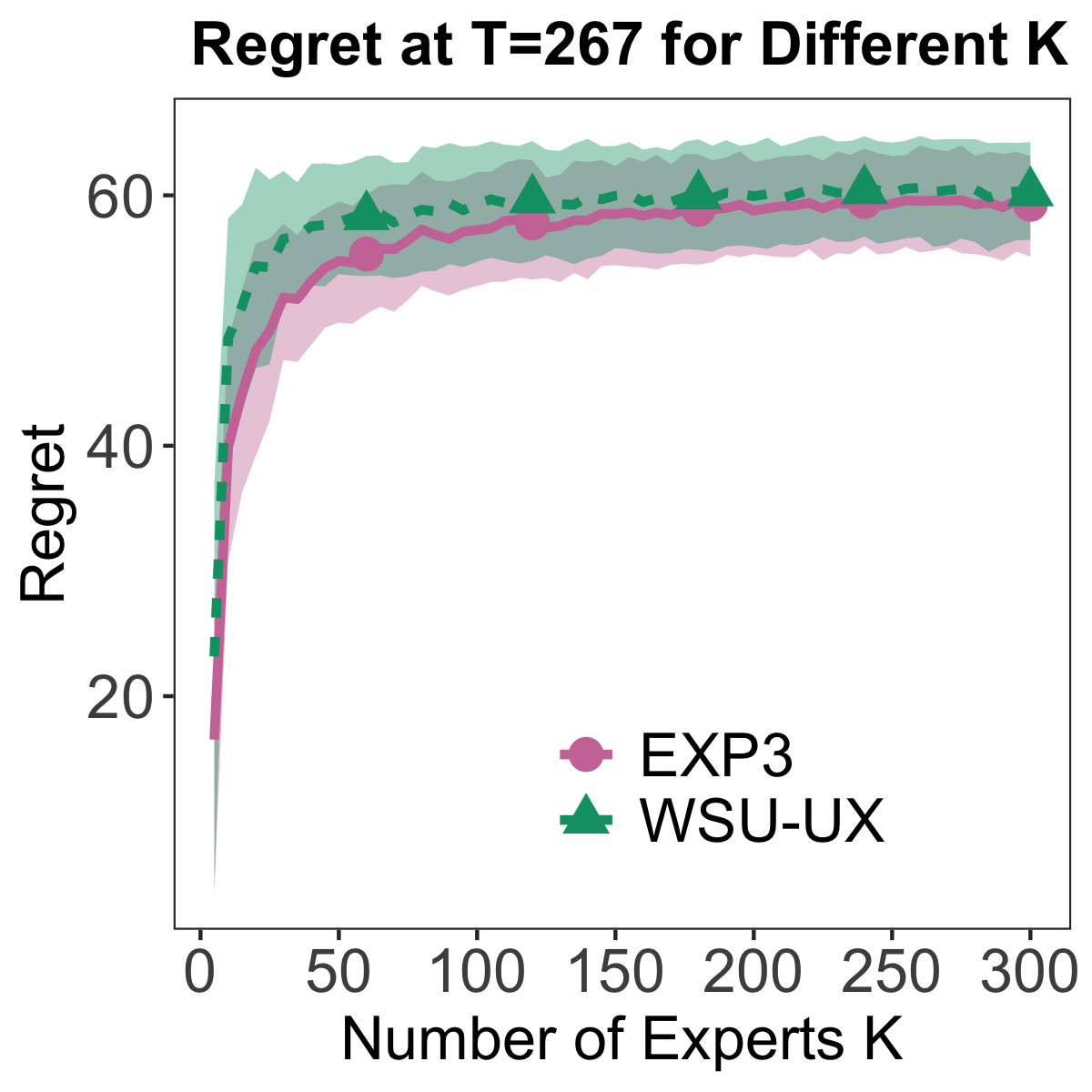}}
\caption{Comparisons on the 2019--2020 FiveThirtyEight NFL dataset. Top: Full-information setting with $\bp_t$ the prediction of a single expert chosen according to $\bpi_t$. Middle: Full-information setting with $\bp_t = \sum_{i \in [K]} \pi_{i,t}p_{i,t}$. Bottom: Partial information setting.}
\label{fig:nfl19-20}
\end{figure*}

We now provide a sketch proof of Theorem~\ref{thm:elfx-ic}, that $\elfx$ is incentive-compatible for forward-looking experts. For details, we refer the reader to~\citet{WFVPK18}.

\begin{proof}[Proof Sketch of Theorem~\ref{thm:elfx-ic}]
	Incentive compatibility rests on the fact that each expert maximizes his (subjective) probability of being selected as the event winner of any timestep $\tau$ by reporting $p_{i,\tau}=b_{i,\tau}$. This is because an expert's probability of being selected as the winner of event $\tau$ is exactly their payment from participating in a Weighted Score Wagering Mechanism where every expert has wager $1/K$. Further, it is easy to check that an expert $i$ \emph{minimizes} the probability of any other expert $j$ being selected as winner of timestep $\tau$ (according to $i$'s belief $b_{i,\tau}$).
	
	Fix the winners on all timesteps other than $\tau$. Because the winner at each timestep is chosen independently of all other timesteps, it is a dominant strategy for each expert to report his belief $b_{i,\tau}$. Incentive compatibility follows by applying this argument to all timesteps $\tau$. 
\end{proof}